\documentclass{article}

\usepackage[preprint]{neurips_2023}
\makeatletter
\renewcommand{\@noticestring}{Preprint.}
\makeatother

\AtBeginDocument{%
  \newgeometry{textwidth=6.5in,textheight=9in,top=1in,headheight=12pt,headsep=25pt,footskip=30pt}%
}

\usepackage[utf8]{inputenc}
\usepackage[T1]{fontenc}
\usepackage{amsmath,amssymb,amsthm}
\usepackage{mathtools}
\usepackage{booktabs}
\usepackage{multirow}
\usepackage{float}
\usepackage{graphicx}
\usepackage{xcolor}
\usepackage{enumitem}
\usepackage{algorithm}
\usepackage{algorithmic}
\usepackage{tikz}
\usetikzlibrary{positioning,arrows.meta,calc,fit,backgrounds,shapes.geometric}
\usepackage{hyperref}
\usepackage{url}

\newcommand{\E}{\mathbb{E}}
\newcommand{\Var}{\operatorname{Var}}
\newcommand{\Cov}{\operatorname{Cov}}
\newcommand{\R}{\mathbb{R}}
\newcommand{\indep}{\perp\!\!\!\perp}
\newcommand{\Avar}{\mathrm{Avar}}
\newcommand{\NEH}{NEH}
\DeclareMathOperator{\Corr}{Corr}
\DeclareMathOperator*{\Med}{Med}

\newtheorem{assumption}{Assumption}
\newtheorem{remark}{Remark}
\newtheorem{proposition}{Proposition}

\title{Representation Learning for Semiparametric Causal Mediation Analysis\\ under No Essential Heterogeneity}

\author{%
  Roberto Faleh\thanks{Corresponding author: \texttt{roberto-faleh[at]uni-tuebingen.de}.} \\
  Methods Center\\
  University of T\"ubingen, Germany\\
  \And
  Sofia Morelli \\
  Methods Center\\
  University of T\"ubingen, Germany\\
  \And
  Holger Brandt \\
  Methods Center\\
  University of T\"ubingen, Germany\\
}

\begin{document}
\maketitle
\begin{abstract}
We propose a two-stage estimator for structural mediation parameters that combines deep representation learning with G-estimation under the ``No Essential Heterogeneity'' (NEH) assumption. We term the method \textsc{UNIT}, for \emph{Unmeasured-confounding-robust NEH-based Identification with TARNet}. In the first stage, TARNet estimates the heterogeneous effect of a randomized treatment on a mediator by learning a shared covariate representation across treatment arms. The resulting conditional average treatment effect (CATE) estimate provides a plug-in approximation to the heterogeneity-dependent component of the weight function entering the G-estimating equation of \citet{ZhengZhou2015}, which identifies the structural parameters even in the presence of unmeasured mediator--outcome confounding. We show that more accurate first-stage representation learning can yield a more informative plug-in weight and thereby improve the precision of the structural parameter estimator. In simulations with non-Gaussian covariates and nonlinear mediator effects, TARNet weights reduce the Stage-2 standard error of the mediation coefficient (median SE ratios $1.45$--$1.51$ at $n\ge2000$) compared to the classical approach, at no cost to bias or coverage.
\end{abstract}


\section{Introduction}
\label{sec:introduction}

The primary goal of many randomized experiments is to estimate the effect of a treatment on an outcome \citep{HernanRobins2020}.  In psychology, economics, and biology, researchers are often interested not only in estimating overall treatment effects, but also in studying the underlying causal mechanisms and understanding how an intervention exerts its effects \citep{VanderWeele2015}. Such an investigation is especially relevant when the objective is, for instance, to evaluate a scientific theory or to guide and inform the design of future interventions. A frequently studied setting \citep{Rijnhart2021,holland1988} involves a mediator~$M$, an intermediate variable that is influenced by the treatment and, in turn, affects the outcome, potentially together with a set of baseline covariates.
Randomization of the treatment $T$ removes confounding between the treatment and the outcome. However, for mediator variables, the situation becomes more complex. Randomization of the treatment does not, in general, imply randomization of the mediator, since the mediator may be influenced by processes that occur after the initial treatment assignment \citep{bullock2010}. The consequence is that the causal effect of the mediator on the outcome can be confounded. In this context, several technical approaches are available and, in general, this does not represent a major methodological concern \citep{Wodtke_Zhou_2026}. A more serious problem arises when some of these confounding variables are not measured \citep{TenHave2010, ImaiKeeleYamamoto2010}.
This is a plausible scenario when, for example, the trial was originally designed to estimate only the overall treatment effect and was not designed to support a causal mediation analysis. In other circumstances, the confounders can be potentially known but unethical, prohibitively expensive, or difficult to measure \citep{Brandt2020}. As a result, researchers could have extensive knowledge about the causal mechanism and a rich set of covariates but still lack confounders needed for the identification of the indirect effects.
For example, an online retail platform may be interested in examining the impact of a banner on customers’ spending, with this relationship potentially mediated by user engagement. The platform could collect customer-level data, including individual preferences and demographic information, yet still lacks covariates necessary for decomposition of the total causal effect. These unobserved variables may include sensitive attributes or information that is  difficult to collect, such as the presence of other individuals in the room while the customer is interacting with the platform.

In such cases, the classical assumption of Sequential Ignorability (SI)\citep{ImaiKeeleYamamoto2010}, which requires the absence of unmeasured mediator--outcome confounding conditional on observed covariates, is not satisfied. Methods that rely on the sequential ignorability assumption, then, would deliver inconsistent estimates \citep{ImaiKeeleYamamoto2010, TenHave2007}.

An alternative framework that avoids relying on the SI assumption is offered by Structural Nested Models (SNMs) together with the corresponding G-estimation approach introduced by \citet{Robins1994}. This approach offers a framework for estimating causal effects in settings where confounders may themselves be influenced by the treatment, as in the case of confounders measured post-treatment. SNMs have considerable advantages over other methods for handling time-varying and unmeasured confounding \citep{loh2025estimating}: they can accommodate flexible structural specification, have doubly robust properties, and can exploit the instrumental variable approach for identification \citep{VansteelandtJoffe2014}. Despite these strengths, the application of G-estimation in empirical research has been relatively infrequent, historically generally due to the complexity of implementation and the limited availability of "off-the-shelf" software \citep{VansteelandtJoffe2014}.
In the causal mediation context, \citet{TenHave2007} applied G-estimation under a rank-preserving assumption. This assumption requires that the idiosyncratic error does not depend on the counterfactual regime at the individual level, meaning that each person's deviation from the structural model is identical regardless of what treatment and mediator values are hypothetically assigned. 

While this allows the existence of unmeasured confounders in the mediator--outcome mechanism, it rules out individual-level heterogeneity in treatment or mediator effects through unobserved confounders.  In many settings this assumption could be difficult to satisfy and often it is an excessively strong constraint:it implies, for instance, that a person whose outcome deviates unusually from the structural prediction under one treatment--mediator regime must deviate by exactly the same amount under every other regime. Indeed, not even non-confounding, idiosyncratic effect heterogeneity is admitted under rank preservation. In many behavioural and biomedical settings, such rank-preservation is implausible and has been criticised accordingly \citep{HernanRobins2020}.

To overcome this limitation, \citet{ZhengZhou2015} propose a more general structural mean model in which the \citet{TenHave2007} formulation appears as a special case.  They alleviate rank preservation by replacing it with a weaker mean-level condition closely related to the econometric ``No Essential Heterogeneity'' (NEH) assumption, building on the marginal treatment effect framework of \citet{HeckmanVytlacil2005} and the definition of essential heterogeneity introduced in \citet{HeckmanUrzuaVytlacil2006}.  This permits arbitrarily rich effect modification through observed covariates, and even idiosyncratic, non-confounding individual heterogeneity, while still accommodating unmeasured confounders that shift both mediator and outcome.

Additionally, the semiparametric structure of the \citet{ZhengZhou2015} model makes it especially suitable for applied settings in which researchers are not merely interested in the magnitude of a treatment effect but also seek to parametrically characterize the causal mechanism through which it operates. 
A practical challenge in implementing the \citet{ZhengZhou2015} framework is the estimation of the optimal weight function that enters the G-estimating equation. \cite{ZhengZhou2015} employed parametric formulations, \cite{Brandt2020} adopted a Thin Plate Spline \citep{Wood2003ThinPlate} to introduce flexibility, and \cite{Morelli2025} broadened the original framework to incorporate latent variables. 

However, when covariate-mediator relationships are nonlinear or the covariate space is high-dimensional, the available tools for weight estimation are limited, and poor weight estimation translates directly into efficiency loss \cite{ZhengZhou2015}. We will show that that the weights depend on the conditional average treatment effect (CATE) of the treatment on the mediator, which must be estimated from the observed data. We will show, both theoretically and with a simulation study, that when the covariate-mediator relationship is nonlinear and high-dimensional, linear approaches may incur in an estimation error for the mediator CATE \citep{curth2021nonparametric}, which in turn degrades the quality of the plug-in weight necessary for the G-estimation step, reducing the precision of the downstream structural parameter estimates.

We address this challenge by proposing a representation-learning  approach \citep{bengiorapler} for the estimation of the weights, thereby broadening the empirical applicability of a framework that already enjoys strong semiparametric identification guarantees. We use a Treatment-Agnostic Representation Network \citep[TARNet;][]{shalit2017}, a neural network architecture with shared representation layers and treatment-specific output heads, to estimate the CATE on the mediator.  The shared representation~$\Phi(X)$ enforces a common feature space across treatment arms. Because the CATE is computed as a difference of two functions of the same representation, the smoothness and dimensionality of~$\Phi$ provide an inductive bias that may help regularise the estimated treatment effect surface, particularly when $\mu_0$ and $\mu_1$, the arm-specific conditional expectations of the mediator, share common features \citep{curth2021nonparametric}.  The TARNet's CATE output is then embedded as the weight function in the G-estimating equation based on the model originally developed by \citet{ZhengZhou2015}. A detailed description is provided in \ref{subsec:from-t-to-tarnet}

\paragraph{Contributions.}
Our contributions are threefold.  First, we illustrate a clear link between the quality of the learned representation~$\Phi$ and the efficiency of the structural parameter estimator: better CATE estimation results in a more informative weight matrix, which reduces the asymptotic variance of~$\hat\theta$.  Second, we propose UNIT (Unmeasured-confounding-robust NEH-based Identification with TARNet), a two-stage algorithm combining cross-fitted TARNet estimation and G-estimation. This approach requires the No Essential Heterogeneity assumption rather than the classical, and stronger, Sequential Ignorability as the key identifying condition.  Third, the practical advantages of this approach are demonstrated through a simulation study with non-Gaussian covariates, complex nonlinear treatment effects, and unmeasured confounding. We show that TARNet-based weights can improve efficiency.

\paragraph{Outline.}
Section~\ref{sec:framework} introduces the causal framework, the model and the NEH assumption.  Section~\ref{sec:methods} develops the representation-learning approach to weight estimation and the two-stage algorithm.  Section~\ref{sec:simulation} presents a simulation study that is used to investigate the properties of the proposed two-stage algorithm under empirically relevant conditions.  Section~\ref{sec:part2-results} presents the results.  Section~\ref{sec:discussion} discusses the findings and limitations.

\newpage
\section{Causal Framework}
\label{sec:framework}

\subsection{Model and potential outcomes}

Following the potential outcomes framework introduced by \citep{Rubin1974} and the model specifications introduced \citet{ZhengZhou2015}, let $Y(t,m)$ denote the potential outcome under treatment level~$t$ and mediator level~$m$, and let $M(t)$ denote the potential mediator under treatment~$t$ \citep{RobinsGreenland1992}.  We observe $n$ independent units $(X_i, T_i, M_i, Y_i)_{i=1}^n$, where $X\in\R^p$ is a vector of pre-treatment covariates, $T\in\{0,1\}$ is the binary treatment, $M\in\R$ is the mediator, and $Y\in\R$ is the outcome.

For observed covariates~$X$ and unobserved confounders~$U$, write
\begin{equation}\label{eq:potential-outcome}
Y(t,m) \;=\; \mu(t,m,X) \;+\; \varepsilon(t,m;\, U,X),
\end{equation}
where $\mu(\cdot)$ captures the systematic (identified) component and $\varepsilon(t,m;\, U,X)$ is an idiosyncratic error satisfying $\E[\varepsilon(t,m;\, U,X)\mid X]=0$ for all $(t,m)$.  The semicolon separates the counterfactual regime $(t,m)$ from the random components $(U,X)$. To streamline the notation, the $(U,X)$ component will be omitted whenever it is not required for the derivations.

We specify a structural mean model of the form
\begin{equation}\label{eq:smm}
\mu(t,m,X) \;=\; g(X) \;+\; \sum_{k=1}^K \theta_k\, h_k(t,m,X),
\end{equation}
where $g(X)$ is a baseline outcome model and the $h_k(\cdot),\; k = 1,2,\ldots,K,$ are known functions, which satisfy $h_k(0,0,X)=0$ and describe the causal effects of the treatment and the mediator.

By consistency assumption (Assumption~\ref{ass:sutva}) the observed outcome satisfies
\begin{equation}\label{eq:observed-outcome}
Y_i \;=\; g(X_i) \;+\; \sum_{k=1}^K \theta_k\, h_k(T_i,M_i,X_i) \;+\; \varepsilon_i,
\end{equation}
where $\varepsilon_i := \varepsilon(T_i, M_i;\, U_i, X_i)$ is the structural error evaluated at the realised treatment and mediator.  Defining the transformed outcome
\[
\tilde Y_i(\theta) \;=\; Y_i \;-\; \sum_{k=1}^K \theta_k\, h_k(T_i,M_i,X_i),
\]
the true parameter~$\theta_0$ satisfies $\E[\tilde Y(\theta_0)\mid T,X] =E[\tilde{Y}(\theta_0)\mid X] =g(X) $ (under assumption \ref{ass:neh}).
For suitable weight functions $A(T,X)$ satisfying $\E[A(T,X)\mid X]=0$, the estimating equation
\begin{equation}\label{eq:gest}
\sum_{i=1}^n A(T_i,X_i)
\left\{Y_i \;-\; g(X_i) \;-\; \sum_{k=1}^K \theta_k\, h_k(T_i,M_i,X_i)\right\} \;=\; 0
\end{equation}
yields a consistent estimator~$\hat\theta$ \citep{ZhengZhou2015}.

Throughout the remainder of this paper, we refer to the model with $K=2$, $h_1(T,M,X) = T$, and $h_2(T,M,X) = M$, as an illustrative example, so that
\begin{equation}\label{eq:rpm-spec}
\mu(t,m,X) - \mu(0,0,X) \;=\; \theta_1\, t \;+\; \theta_2\, m.
\end{equation}
This implies no treatment-by-mediator interaction in the outcome and constant controlled effects $\theta_1$, $\theta_2$. This keeps the structural parameters interpretable and preserves comparability with the results reported in \citet{TenHave2007}, \citet{ZhengZhou2015}, \citet{Brandt2020}, and \citet{Morelli2025}. Richer specifications in the structural model can be easily accommodated \citep{ZhengZhou2015} and additional details will be developed in Section ~\ref{subsec:mediator-cate}.

\subsection{Overview on 2-stage approach}
\label{subsec:setting}

Parsimony in the structural outcome model does not require parsimony in the rest of the data-generating process. Our methodological focus is on settings where the covariate vector~$X\in\R^p$ is high-dimensional and the
mediator CATE
\begin{equation}\label{eq:mediator-cate}
\tau_M(X) \;=\; \E[M\mid T{=}1, X] \;-\; \E[M\mid T{=}0, X]
\end{equation}
is a non-trivial function to learn, involving nonlinearities or
interactions that cannot be captured by simple linear models.
The setting we are describing is the combination of a complex and heterogeneous treatment effect on the mediator, paired with a relatively simple, structural outcome model, and could generally arise in many applied settings. For example, baseline individual characteristics (e.g., genetic profile, comorbidities) may interact in a complex manner in determining the level of a biomarker (e.g., blood pressure), whereas the effect of the biomarker on the outcome (e.g., cardiovascular risk) per unit change in the biomarker may be relatively simple, invariant among the population, and approximately constant.

In the specific formulation made, the weights are 
\begin{equation}\label{eq:optimal-weight-preview}
A^*(T,X) \;=\; \bigl(T-\E[T]\bigr)\,W^*(X), \qquad
W^*(X) \;=\; \begin{pmatrix} 1 \\[2pt] \tau_M(X) \end{pmatrix}.
\end{equation} details about the weights and the CATE will be discussed in Section~\ref{subsec:mediator-cate}.

We refer to this step, the estimation of the weights necessary for the estimating equations, as Stage 1 . Solving the estimating equations is Stage 2.  We summarize the algorithm below; the full development is deferred to Section~\ref{sec:methods}.

\subsection{General Algorithm Overview}
\label{subsec:algoverview}

Collect the basis functions of Eq.~\eqref{eq:smm} into the vector
\[
H_i \;:=\; \bigl(h_1(T_i,M_i,X_i),\,\dots,\,h_K(T_i,M_i,X_i)\bigr)^\top \in \R^K,
\]
so that the structural component of the model is
$\sum_{k=1}^K \theta_k\,h_k(T_i,M_i,X_i) = H_i^\top\theta$, with
$\theta=(\theta_1,\dots,\theta_K)^\top$. In the structural mean (mediation) model the
basis becomes $H_i=(T_i,M_i)^\top$ and $\theta=(\theta_1,\theta_2)^\top$ are the
controlled direct and mediator effects.

\paragraph{Stage 1: weight estimation.}
We estimate the optimal weight $W^*(X)=(1,\tau_M(X))^\top$
by replacing the unknown mediator CATE with a cross-fitted estimate
$\hat\tau_M^{(-)}$ of~\eqref{eq:mediator-cate}. Any CATE learner may be used; we use a
cross-fitted TARNet (Section~\ref{subsec:tarnet}). This yields the plug-in centred
instrument
\[
\hat A_i \;=\; (T_i-p)\,\hat W(X_i), \qquad
\hat W(X_i)=\bigl(1,\;\hat\tau_M^{(-)}(X_i)\bigr)^\top,
\]
which satisfies the centering condition $\E[\hat A_i\mid X_i]=0$ under randomisation. The
weight $\hat A_i$ is held fixed for Stage~2.

\paragraph{Stage 2: structural estimation.}
Given the Stage~1 weight $\hat A_i$ and a cross-fitted estimate $\hat g^{(-)}$ of the
baseline, the structural estimator $\hat\theta$ is defined as the solution of the
empirical G-estimating equation in Eq.~\eqref{eq:gest}:
\begin{equation}\label{eq:gest-emp}
\sum_{i=1}^n \hat A_i\,\bigl\{\,Y_i - \hat g^{(-)}(X_i) - H_i^\top\theta\,\bigr\} \;=\; 0 .
\end{equation}

It is important to highlight two aspects about Stage~2;

First, because $\dim\hat A_i=\dim H_i=K$, Eq.~\eqref{eq:gest-emp} is a
just-identified linear system in~$\theta$, with closed-form solution
\begin{equation}\label{eq:theta-hat}
\hat\theta \;=\; \Bigl(\textstyle\sum_{i=1}^n \hat A_i H_i^\top\Bigr)^{-1}
\sum_{i=1}^n \hat A_i\,\bigl\{Y_i - \hat g^{(-)}(X_i)\bigr\}.
\end{equation}
This is an instrumental-variable estimator: the centred weight $\hat A_i$ acts as an
instrument for the structural regressors $H_i$, and invertibility of
$\sum_i \hat A_i H_i^\top$ is exactly the relevance (non-degeneracy) condition of
Assumption~\ref{ass:nondegen}.

Second, because the instrument is centred, the equation is orthogonal to the baseline:
$\E[\hat A_i\,f(X_i)]=0$ for any function $f$ of the covariates. Consistency of
$\hat\theta$ therefore does not require $\hat g^{(-)}$ to be correctly specified; the
baseline enters only as a variance-reducing nuisance (cf.\ the working-model property
of Eq.~\eqref{eq:gest}).

The baseline is itself fit by cross-fitting on the transformed outcome
$\tilde Y_i(\theta)=Y_i-H_i^\top\theta$, and so depends on the current value
of~$\theta$. In the general case $\hat g^{(-)}$ and $\hat\theta$ are thus determined
jointly, by the estimation scheme of \citet{ZhengZhou2015}. Under the linear working
model for the baseline adopted here the joint system is linear and Eq.~\eqref{eq:theta-hat}
is obtained in a single step, consistent with the one-step recommendation of
\citet{ZhengZhou2015}; the concrete cross-fitted solver and this reduction are given in
Algorithm~\ref{alg:unit}.

\subsection{ Causal Assumptions}

The model rests on five assumptions, introduced below.  The first two: consistency/SUTVA
(Assumption~\ref{ass:sutva}) and treatment ignorability with
positivity (Assumption~\ref{ass:ignorability}) are standard in the causal inference literature and
hold by design in a randomised trial.  The third, no essential
heterogeneity (Assumption~\ref{ass:neh}), is the key identifying
condition that distinguishes our approach from methods based on
sequential ignorability: it permits unmeasured mediator--outcome
confounding at the cost of restricting how unobserved confounders may
modify treatment and mediator effects in the mean.  
The final two assumptions, which are of a statistical nature, are necessary for estimation. In particular, the fourth assumption, concerning the specification of the structural model (Assumption~\ref{ass:model-spec}), is necessary to guarantee consistency. The 
fifth assumption, the 
non-degeneracy of the weight matrix
(Assumption~\ref{ass:nondegen}), ensures that the structural
parameters are identifiable from the estimating equation.

\begin{assumption}[Consistency / SUTVA]\label{ass:sutva}
\mbox{}
\begin{enumerate}
    \item[(i)] (No interference) The potential outcomes and potential mediators of unit~$i$ do not depend on the treatment or mediator values assigned to other units.
    \item[(ii)] (No hidden treatment versions) There is a single version of each treatment level, so that the mapping from treatment assignment to potential outcome is well-defined.
    \item[(iii)] (Consistency) Observed outcomes equal the relevant potential outcomes under realised treatment and mediator values: $M_i = M_i(T_i)$ and $Y_i = Y_i(T_i, M_i)$.

\end{enumerate}
\end{assumption}
The three conditions constituting Assumption~\ref{ass:sutva} arise directly from the experimental design \citep{Morgan_Winship_2014}. Their validity must be evaluated prior to the initiation of the experiment, as any violation would render both estimation and identification infeasible.

\begin{assumption}[Treatment ignorability and positivity]\label{ass:ignorability}
\mbox{}
\begin{enumerate}
    \item[(i)] (Ignorability) Treatment assignment is independent of all potential outcomes and potential mediators given pre-treatment covariates:
    \[
    \big(M(0),\,M(1),\,Y(t,m):t,m\big)\;\indep\;T\;\mid\;X.
    \]
    \item[(ii)] (Positivity) Every unit has a positive probability of receiving each treatment level: $0 < \Pr(T=t\mid X) < 1$ for each $t\in\{0,1\}$.
\end{enumerate}
\end{assumption}

 In observational settings, Assumption~\ref{ass:ignorability} would require that all confounders of the $T-M$ and $T-Y$ relationships are measured and included in $X$. The extension of UNIT to this case is discussed in Section ~\ref{sec:limitations} .
It should be noted that Assumption~\ref{ass:ignorability} concerns only the treatment assignment mechanism. It does not require the absence of unmeasured confounding between~$M$ and~$Y$.

\begin{assumption}[No Essential Heterogeneity]\label{ass:neh}
For all treatment and mediator levels $t,m$,
\begin{equation}\label{eq:neh-error}
\E\!\left[\varepsilon(t,m;\, U,X)\,\middle|\,M,X,T\right] \;=\; F(M,X,T),
\end{equation}
where $F$ is the \emph{same function} of $(M,X,T)$ across all counterfactual regimes $(t,m)$.
\end{assumption}

Equivalently, for any $(t_1,m_1)$ and $(t_0,m_0)$,
\begin{equation}\label{eq:neh-diff}
\E\!\left[Y(t_1,m_1)-Y(t_0,m_0)\,\middle|\,M,X,T\right]
\;=\;\mu(t_1,m_1,X)-\mu(t_0,m_0,X),
\end{equation}
so that differences in conditional means across regimes are driven entirely by the systematic component~$\mu$.  Assumption~\ref{ass:neh} allows for rich heterogeneity through~$X$ (including $X$-specific effect modification via~$\mu$) and arbitrary residual noise. However, it excludes unmeasured effect modification in the mean: once we condition on $M, T, X$, unobserved confounders do not influence the average gain from switching $t, m$. From an econometric perspective, this is the equivalent of the negation of ``Essential Heterogeneity'', as formally defined by \citet{HeckmanUrzuaVytlacil2006}.

We emphasize that, although the NEH assumption is weaker than SI  \citep{Small2012,ZhengZhou2015}, it is suboptimal in settings where SI actually holds. This loss of optimality manifests as an increased variance in the resulting parameter estimates relative to estimators that target the SI semiparametric efficiency bound (e.g., \citealp{TchetgenShpitser2012}). Moreover, if SI holds, nonparametric identification is possible, adding flexibility to the estimation procedure \citep{ImaiKeeleYamamoto2010,Pearl2001}.

\begin{assumption}[Structural model specification]\label{ass:model-spec}
The basis functions $\{h_k(t,m,X)\}_{k=1}^K$ in the structural mean model in Eq.~\eqref{eq:smm} are assumed to be correctly specified up to unknown parameters, the nuisance function $g(\cdot)$ can be misspecified. 

\end{assumption}
This assumption is only partially testable: since potential outcomes are not jointly observed, and only the corresponding model for the observed outcome can be examined empirically.

Assumption~\ref{ass:model-spec} is necessary for consistency. If structural terms are omitted, the baseline model $g(X)$ can not absorb them, introducing bias in the estimated causal parameters~$\hat\theta$.   On the other hand, if $g$ is incorrectly specified, only efficiency is affected, but not consistency. This is the working-model property of G-estimation (see the centering argument in Appendix~\ref{app:rep-gap}, Eq.~\eqref{eq:centering-key}).

To estimate the model, an additional assumption regarding the regularity properties of the weight matrix needs to be introduced.
\begin{assumption}[Non-degeneracy of the weight matrix]\label{ass:nondegen}
The weight vector $W(X)$ is non-degenerate such that its second moment matrix $\E[W(X)W(X)^\top]$ is positive definite.
\end{assumption}

For example the weight vector introduced for the illustrative model earlier is $W(X)= (1, \tau_M(X)) ^\top$. In particular, we require that $\Var(\tau_M(X)) > 0$, implying that the treatment has a heterogeneous effect on the mediator across the covariate space.  This suggests that if the treatment has no heterogeneous effects on the mediator (i.e., $\tau_M(X)$ is constant), the second row of the weight matrix is collinear with the first, then the estimating equation has no unique solution for~$\theta_2$, and the mediation parameter is not identified. Assumption~\ref{ass:nondegen} is the mediation analogue of the instrumental-variable relevance condition:  this condition can be assessed empirically by testing whether treatment-by-covariate interactions in the mediator model are significant \citep{Morelli2025}.

In practice, $\tau_M(X)$ is unknown and must be estimated from the data.  Even if the model is formally identified with $\Var(\tau_M(X)) > 0$, the actual size of $\Var(\tau_M(X))$ affects the efficiency of the estimator in Stage-2; this is the task we address in Section~\ref{sec:methods}.

\subsection{The mediator CATE and the optimal weight function}
\label{subsec:mediator-cate}

The G-estimating Eq.~\eqref{eq:gest} holds for any weight
function $A(T,X)$ satisfying the centering condition
$\E[A(T,X)\mid X] = 0$.  Different choices of~$A$ yield
consistent estimators with different asymptotic variances.
Theorem~3 of \citet{ZhengZhou2015} characterises the
semiparametrically efficient choice: the optimal weight is
\[
A^*(T,X) \;=\;
\bigl\{\E[\partial\tilde Y/\partial\theta \mid X,T]
     - \E[\partial\tilde Y/\partial\theta \mid X]\bigr\}
\;\Omega^{-1}(X),
\]
where $\Omega(X) = \Var(\varepsilon(T,M;\,U,X)\mid X)$.  UNIT adopts the $\Omega^{-1}\equiv \mathbf{I}_p$ simplification (valid under the linear-$g$ working model), so the deployed weight is $(T-p)\,W^*(X)$; see Appendix~\ref{subsec:omega-role}.
For our illustrative model specification
the partial derivative $\partial\tilde Y/\partial\theta$ equals
$-H_i$ with $H_i = (T_i, M_i)^\top$, and the difference of
conditional expectations of $H_i$ simplifies to
\[
\E[H_i \mid X, T] - \E[H_i \mid X]
\;=\;
(T - p)\begin{pmatrix} 1 \\[2pt] \tau_M(X) \end{pmatrix},
\]
where the first entry is trivial ($\E[T\mid X,T] - \E[T\mid X]
= T - p$) and the second follows from
$\E[M\mid X,T=1] - \E[M\mid X] = (1-p)\,\tau_M(X)$ and
analogously for $T=0$.  The optimal weight therefore takes the
multiplicative form
\begin{equation}\label{eq:optimal-weight}
A^*(T,X) \;=\; (T - p)\,W^*(X)\,\Omega^{-1}(X),
\qquad
W^*(X) \;=\; \begin{pmatrix} 1 \\[2pt] \tau_M(X)
              \end{pmatrix}.
\end{equation}

Under Assumptions~\ref{ass:sutva}
and~\ref{ass:ignorability}, the mediator CATE is identified
from observed data:
\begin{equation}\label{eq:cate-identification}
\tau_M(X) \;=\; \E[M\mid T{=}1,X] \;-\; \E[M\mid T{=}0,X].
\end{equation}

Two features of Eq.~\eqref{eq:optimal-weight} need additional details.  First, the weight vector $W^*(X)$ depends on~$X$
only through the mediator CATE $\tau_M(X)$.  This means that
the entire problem of constructing efficient instruments
for~$\theta$ reduces to the problem of estimating a single
conditional average treatment effect, which is a well-studied task in
the causal inference literature \citep{Kunzel2019}.  Second, the centering factor
$(T-p)$ ensures that $\E[A^*(T,X)\mid X]=0$ automatically
under randomisation, regardless of the quality of
$\hat\tau_M$.  Consequently, any plug-in weight
$\hat A_i = (T_i - p)\hat W(X_i)$ preserves the consistency
of~$\hat\theta$; the quality of $\hat\tau_M$ affects only
efficiency, not identification.

Because $\tau_M(X)$ is a CATE, with the mediator~$M$
playing the role of outcome and the randomised treatment~$T$
as the intervention, the rich literature on CATE estimation
applies directly to Stage~1 of our algorithm.  
Existing theoretical results on learner-specific error bounds and the function of shared representations \citep{curth2021nonparametric}, systematic comparisons of metalearners \citep{Kunzel2019}, the use of cross-fitting for nuisance-function estimation \citep{Chernozhukov2018}, as well as analyses of the dependence of learner performance on the underlying data-generating process \citep{curth2021nonparametric}, are all directly pertinent to our setting. These considerations motivate the adoption of representation-learning architectures \citep{pmlr-v48-johansson16}, such as TARnet \citep{shalit2017}.

For richer structural models, additional structural components can be added.  If the basis includes
a treatment-by-mediator interaction $h_3 = TM$, the corresponding
row of the optimal weight involves the arm-specific conditional
mean $\mu_1^M(X) = \E[M\mid T{=}1, X]$, not just the difference.  More generally, for any basis function that is
linear in~$M$ without $T\times M$ interaction, including
covariate modulated effects such as $h_k = M\cdot f(X)$, the
weight continues to depend on~$X$ through $\tau_M(X)$ alone.
In all cases where arm-specific means are required, TARNet
provides them directly through its two output heads
$h_0(\Phi(X))$ and $h_1(\Phi(X))$, so the extension of UNIT
to richer structural specifications does not require a change
of Stage~1 architecture, only the construction of the weight
vector~$W(X)$ and the basis vector~$H_i$ must be adapted accordingly.

\section{Methods: Representation Learning for Weight Estimation}
\label{sec:methods}

\subsection{From separate regressions to shared representations}
\label{subsec:from-t-to-tarnet}

The most direct approach to estimating the mediator CATE is the T-learner \citep{Kunzel2019}: fit two separate response surfaces $\widehat{\mu}_0(x)$ and $\widehat{\mu}_1(x)$ on the control and treated groups respectively, and define $\widehat{\tau}_{T}(x) = \widehat{\mu}_{1}(x) - \widehat{\mu}_{0}(x)$.  While consistent, the T-learner can exhibit larger finite-sample estimation error than methods that share information across treatment arms.  Because $\widehat{\mu}_{0}$ and $\widehat{\mu}_{1}$ are estimated on disjoint subsets, each model is regularised independently, and when the two groups differ in size, the regularisation schedules diverge.  This regularisation mismatch produces variation in $\widehat{\tau}_{T}(x)$ that does not correspond to treatment-effect heterogeneity \citep{Kunzel2019}.  More formally, the error bound for the T-learner scales as $2\{\varepsilon(\hat\mu_1) + \varepsilon(\hat\mu_0)\}$, where $\varepsilon(\hat\mu_w)$ denotes the estimation error of each of the $w=0,1$ arm-specific model \citep{curth2021nonparametric}.  Learners that share structure between arms can achieve tighter bounds by exploiting a common signal \cite{curth2021nonparametric}.

A more fundamental limitation is that the T-learner cannot exploit structural similarities between the two response surfaces.  In many applications, the conditional expectations $\mu_1(x)$ and $\mu_0(x)$ share substantial low-dimensional structure; the treatment shifts the outcome without radically altering which features are relevant \citep{Kennedy2023}.  Separate models must therefore recover the common signal independently in each arm, even though that signal is largely shared, and under balanced 1:1 treatment allocation, each model uses at most half of the available data.
These considerations motivate the use of representation learning as an inductive bias for CATE estimation.  Instead of separate models, we learn a shared feature map $\Phi\colon\R^p\to\R^r$ that compresses the covariate space into a common representation for both treatment arms.  Treatment-specific heads $h_0$ and $h_1$ then map from this representation to the arm-specific conditional expectations.  The key structural assumption, formalised by \citet{shalit2017}, is that there exists a common feature space underlying both $\{\mu_w(x)\}_{w\in\{0,1\}}$, so that $\Phi(X)$ preserves all identifying conditional independence relationships \citep{curth2021nonparametric, shalit2017}.

By pooling all observations through the shared layers, the representation~$\Phi$ is estimated with the full sample, regardless of treatment assignment.  This yields two advantages directly relevant for the G-estimation framework.  First, when treatment groups are of unequal size, the shared layers stabilise estimation by pooling information across arms \citep[variance reduction under imbalance;][]{curth2021nonparametric}.  Second, because the CATE is computed as a difference of two functions of the same representation, the smoothness and dimensionality of~$\Phi$ provide an inductive bias that may help regularise the treatment effect surface, particularly when $\mu_0$ and $\mu_1$ share substantial low-dimensional structure \citep{Kennedy2023,shalit2017}.



\subsection{Direct TARNet estimation of the mediator CATE}
\label{subsec:tarnet}

We adopt a TARNet \citep{shalit2017} as implemented in the SNet-1 architecture of \citet{curth2021nonparametric}.  The network consists of three components:
\begin{enumerate}[label=(\roman*)]
\item A \textbf{shared representation} $\Phi\colon\R^p\to\R^r$, parameterised as a multi-layer feed-forward network with ELU activations, mapping the raw covariate vector~$X$ into a latent feature space.

\item Two \textbf{treatment-specific heads} $h_0,h_1\colon\R^r\to\R$, each a smaller dense network, predicting the arm-specific conditional expectations:
\[
\hat m_0(x)=h_0\!\bigl(\Phi(x)\bigr),\qquad
\hat m_1(x)=h_1\!\bigl(\Phi(x)\bigr).
\]

\item The \textbf{mediator CATE} is obtained directly as the contrast:
\begin{equation}\label{eq:tau-tarnet}
\hat\tau_M(x) \;=\; h_1\!\bigl(\Phi(x)\bigr) \;-\; h_0\!\bigl(\Phi(x)\bigr).
\end{equation}
\end{enumerate}

The entire network is trained end-to-end by minimising the factual loss
\begin{equation}\label{eq:tarnet-loss}
\widehat{\mathcal{L}}_{\mathrm{TARNet}}
\;=\;
\frac{1}{n}\sum_{i=1}^n
\left\{M_i - \bigl[(1-T_i)\,h_0\!\bigl(\Phi(X_i)\bigr)
+ T_i\,h_1\!\bigl(\Phi(X_i)\bigr)\bigr]\right\}^2
+ \lambda\,\mathcal{R}(h_0,h_1,\Phi),
\end{equation}
where $\mathcal{R}(\cdot)$ is an $L_2$-regularisation penalty on all network weights.  Each observation contributes through exactly one head (corresponding to its observed treatment), while both treated and control units contribute to the shared representation~$\Phi$.
\newline

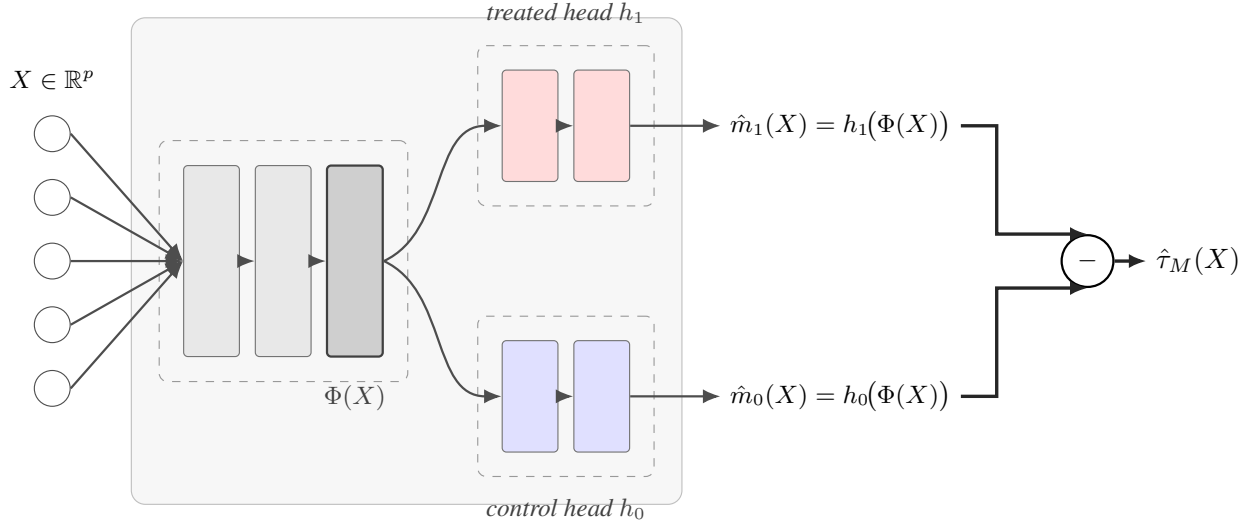
\begin{figure}[H]
\centering
\resizebox{\linewidth}{!}{%
\begin{tikzpicture}[
    >={Latex[length=2.2mm,width=1.8mm]},
    font=\small,
    neuron/.style          ={circle, draw=black!70, minimum size=4.5mm, inner sep=0pt, fill=white},
    hiddenlayer/.style     ={rectangle, draw=black!55, rounded corners=1.5pt,
                             minimum width=7mm, minimum height=24mm, fill=gray!18},
    replayer/.style        ={rectangle, draw=black!75, thick, rounded corners=1.5pt,
                             minimum width=7mm, minimum height=24mm, fill=gray!38},
    headlayer/.style       ={rectangle, draw=black!55, rounded corners=1.5pt,
                             minimum width=7mm, minimum height=14mm, fill=blue!12},
    headlayertreated/.style={headlayer, fill=red!14},
    iolabel/.style         ={font=\small},
    op/.style              ={circle, draw=black, thick, minimum size=6.5mm, inner sep=0pt, fill=white},
    grouplabel/.style      ={font=\footnotesize\itshape, text=black!75},
    groupbox/.style        ={draw=black!40, dashed, rounded corners=3pt, inner sep=3mm},
    trainbox/.style        ={draw=black!25, rounded corners=5pt, inner sep=6.5mm, fill=black!3},
    arr/.style             ={->, thick, black!70},
    arrth/.style           ={->, very thick, black!85},
]

\foreach \y/\name in {1.6/x1, 0.8/x2, 0/x3, -0.8/x4, -1.6/x5}
    \node[neuron] (\name) at (0,\y) {};
\node[iolabel, above=2mm of x1] {$X \in \mathbb{R}^{p}$};

\node[hiddenlayer] (s1) at (2.0, 0) {};
\node[hiddenlayer] (s2) at (2.9, 0) {};
\node[replayer]    (s3) at (3.8, 0) {};

\node[headlayertreated] (h1a) at (6.0,  1.7) {};
\node[headlayertreated] (h1b) at (6.9,  1.7) {};
\node[headlayer]        (h0a) at (6.0, -1.7) {};
\node[headlayer]        (h0b) at (6.9, -1.7) {};

\node[iolabel] (m1) at (9.9,  1.7) {$\hat m_1(X)=h_1\!\bigl(\Phi(X)\bigr)$};
\node[iolabel] (m0) at (9.9, -1.7) {$\hat m_0(X)=h_0\!\bigl(\Phi(X)\bigr)$};

\node[op] (sub) at (13.0, 0) {$-$};
\node[iolabel, right=4mm of sub, font=\normalsize] (tau) {$\hat\tau_M(X)$};

\foreach \i in {1,2,3,4,5}{ \draw[arr] (x\i.east) -- (s1.west); }
\draw[arr] (s1) -- (s2);
\draw[arr] (s2) -- (s3);
\draw[arr] (s3.east) to[out=25,in=180]  (h1a.west);
\draw[arr] (s3.east) to[out=-25,in=180] (h0a.west);
\draw[arr] (h1a) -- (h1b);
\draw[arr] (h0a) -- (h0b);
\draw[arr] (h1b.east) -- (m1.west);
\draw[arr] (h0b.east) -- (m0.west);
\draw[arrth] (m1.east) -- ++(0.45,0) |- (sub.north);
\draw[arrth] (m0.east) -- ++(0.45,0) |- (sub.south);
\draw[arrth] (sub.east) -- (tau.west);

\begin{scope}[on background layer]
  \node[trainbox, fit=(s1)(h1a)(h1b)(h0a)(h0b)] (trainregion) {};
  \node[groupbox, fit=(s1)(s2)(s3)] (sharedbox) {};
  \node[groupbox, fit=(h1a)(h1b)]   (h1box) {};
  \node[groupbox, fit=(h0a)(h0b)]   (h0box) {};
\end{scope}

\node[grouplabel] at ($(s1)!0.5!(s2) + (0,-1.72)$) { };
\node[grouplabel] at ($(s3) + (0,-1.72)$) {$\Phi(X)$};
\node[grouplabel, above=1.5mm of h1box] {treated head $h_1$};
\node[grouplabel, below=1.5mm of h0box] {control head $h_0$};
\node[grouplabel, anchor=north west, text=black!55]
      at ([xshift=1mm,yshift=-1mm]trainregion.north west)
      {};
\end{tikzpicture}%
}
\caption{TARNet (SNet--1) architecture for estimating the mediator CATE
$\hat\tau_M(X)$. The covariate vector $X\!\in\!\mathbb{R}^{p}$ passes through a
stack of shared layers whose final layer is the representation $\Phi(X)$. This
representation feeds two treatment-specific heads, $h_0$ for control units and
$h_1$ for treated units. The shaded region is trained end-to-end on the factual
loss in Eq.~\eqref{eq:tarnet-loss}: each observation contributes only through the head
matching its observed treatment, while the shared layers are updated from all
units. The mediator CATE in Eq.~\eqref{eq:tau-tarnet} is obtained at evaluation time
as the contrast $\hat\tau_M(X)=h_1\!\bigl(\Phi(X)\bigr)-h_0\!\bigl(\Phi(X)\bigr)$.}
\label{fig:tarnet-arch}
\end{figure}

\subsection{Cross-fitting both nuisance components}
\label{subsec:cross-fitting}

To separate nuisance estimation from score evaluation, we cross-fit both nuisance components entering Stage~2: the mediator CATE $\tau_M(X)$ and the baseline outcome function~$g(X)$.
We model the nuisance baseline as a linear working model, $g(X)=\beta^\top b(X)$, where $b(X)$ is a fixed basis expansion that may include polynomial terms of the raw covariates.  Thus $g$ is linear in its coefficients even when it is polynomial in the original variables, avoiding introducing an additional nonlinear nuisance layer in Stage~2.

In principle, $g(x)$ can be nonlinear and specified more flexibly, for example with a second neural network. The main consequences would be two. First, the baseline nuisance is fit to the transformed
outcome
\[
  \tilde Y_i(\theta)=Y_i-\theta_1 T_i-\theta_2 M_i,
\]
so the target of the baseline regression changes whenever $\theta$ is
updated. This implies that it becomes necessary to adopt a general estimation strategy: alternating update of $g$ and $\theta$ \citep{ZhengZhou2015}.  The second consequence is that it becomes necessary to estimate the $\Omega^{-1}(X)$ component. Under our linear in-parameters specification both complications collapse, the update of $\theta$ reduces to
a one-step G-estimation equation, and $\Omega^{-1}(X)$ need not be
estimated explicitly. However, we present the alternating algorithm, as it is the more general method and remains applicable when $g(x)$ is extended beyond the linear working model $b(x)$. A more detailed discussion of these aspects is provided in \ref{subsec:omega-role}.
  Let $\{I_k\}_{k=1}^K$ be a partition of the sample into $K$ folds, and let $I_k^c$ denote the complementary training set.

\paragraph{Stage~1 nuisance: mediator CATE.}
For each fold $k=1,\dots,K$, we train a TARNet on $I_k^c$ by minimising the factual loss in Eq.~\eqref{eq:tarnet-loss}, obtaining fold-specific estimators $\Phi^{(-k)}, h_0^{(-k)}, h_1^{(-k)}$.  For every unit $i \in I_k$, we compute the out-of-fold mediator CATE
\[
\hat\tau_M^{(-)}(X_i)
\;=\;
h_1^{(-k)}\!\bigl(\Phi^{(-k)}(X_i)\bigr)
\;-\;
h_0^{(-k)}\!\bigl(\Phi^{(-k)}(X_i)\bigr).
\]
Stacking across folds yields cross-fitted estimates $\hat\tau_M^{(-)}(X_i)$ for all~$i$.

\paragraph{Stage~2 nuisance: baseline outcome model.}
The baseline function is estimated foldwise inside Stage~2 as well.  Given a current iterate~$\theta^{(t)}$, define the transformed outcome $\tilde Y_i(\theta^{(t)}) = Y_i - H_i^\top\theta^{(t)}$ with $H_i = (T_i, M_i)^\top$.  For each fold~$k$, fit the baseline learner on $I_k^c$ using $\{\tilde Y_i(\theta^{(t)}), X_i : i \in I_k^c\}$, and denote the resulting estimator by $\hat g_{\theta^{(t)}}^{(-k)}$.  For each held-out observation $i\in I_k$, define the out-of-fold prediction $\hat g_{\theta^{(t)}}^{(-)}(X_i) = \hat g_{\theta^{(t)}}^{(-k)}(X_i)$.

Cross-fitting is relevant for the asymptotic argument in Appendix~\ref{app:inference}. In particular, when the same sample is used both to estimate~$g$ and to evaluate the score, the empirical process remainder from plugging in~$\hat g$ requires additional regularity conditions \citep[e.g., Donsker-class membership;][]{Chernozhukov2018} that may not hold for flexible learners.

Cross-fitting the baseline removes this dependence: conditional on the training folds, the held-out score contributions are mean-zero, and the baseline remainder becomes of order $\|\hat g_\theta^{(-)} - g_{0,\theta}\|_{L_2(P)}$, which is negligible under consistency \citep{Chernozhukov2018}. 
\subsection{G-estimation weights}
\label{subsec:weights}

We now explicitly specify how the CATE on the mediator output of the TARNet model is incorporated into the G-estimation equation. Substituting the cross-fitted estimates into Eq.~\eqref{eq:optimal-weight}, the plug-in weight vector for unit~$i$ is
\begin{equation}\label{eq:What}
\hat W(X_i) \;=\; \begin{pmatrix} 1 \\[3pt] \hat\tau_M^{(-)}(X_i) \end{pmatrix}
\;=\; \begin{pmatrix} 1 \\[3pt] h_1^{(-k(i))}\!\bigl(\Phi^{(-k(i))}(X_i)\bigr) - h_0^{(-k(i))}\!\bigl(\Phi^{(-k(i))}(X_i)\bigr) \end{pmatrix},
\end{equation}
and the empirical optimal weight becomes
\begin{equation}\label{eq:Ahat}
\hat A(T_i,X_i) \;=\; \bigl(T_i - p\bigr)\,\hat W(X_i),
\end{equation}
where $p$ is the known allocation probability (or $\bar T$ if estimated).

The representation~$\Phi$ determines the quality of $\hat\tau_M(X)$ and thereby the informativeness of the weight matrix~$\hat A$.  If~$\Phi$ captures all relevant heterogeneity in the $T\to M$ effect, the resulting weights provide a good plug-in approximation to the heterogeneity component of the semiparametrically efficient instrument \citep[cf.\ Theorem~3 in][]{ZhengZhou2015}, thereby improving the efficiency of~$\hat\theta$. Conversely, if~$\Phi$ is chosen too coarsely and discards features that generate true effect heterogeneity, the weights lose informative signal and $\hat\theta$ becomes less efficient, although it still remains consistent under the NEH assumption.

\begin{algorithm}[th]
\caption{UNIT: Cross-fitted two-stage G-estimation}\label{alg:unit}
\begin{algorithmic}[1]
\REQUIRE Data $\mathcal{D} = \{Y_i, T_i, M_i, X_i\}_{i=1}^n$; number of folds~$K$; propensity~$p$; basis $H_i=(h_1,\dots,h_K)^\top$ of model~\eqref{eq:smm}; baseline class~$\mathcal{G}$; cross-fitted CATE learner; tolerance~$\varepsilon$.
\STATE Partition $\mathcal{D}$ into $K$ folds $\{I_k\}_{k=1}^K$.
\medskip
\STATE \textbf{--- Stage 1: Cross-fitted optimal weight ---}
\FOR{$k = 1, \dots, K$}
    \STATE Train the CATE learner on $I_k^c$; we use TARNet $(\Phi^{(-k)}, h_0^{(-k)}, h_1^{(-k)})$.
    \FOR{$i \in I_k$}
        \STATE $\hat\tau_M^{(-)}(X_i) \leftarrow h_1^{(-k)}(\Phi^{(-k)}(X_i)) - h_0^{(-k)}(\Phi^{(-k)}(X_i))$
    \ENDFOR
\ENDFOR
\STATE Form the plug-in centred instrument: $\hat W(X_i) = (1,\; \hat\tau_M^{(-)}(X_i))^\top$, \quad $\hat A_i = (T_i - p)\,\hat W(X_i)$. \COMMENT{$\hat A_i$ fixed hereafter}
\medskip
\STATE \textbf{--- Stage 2: Cross-fitted G-estimation of $\theta$ ---}
\STATE Goal: solve the estimating equation~\eqref{eq:gest-emp}, $\ \sum_i \hat A_i\{Y_i - \hat g^{(-)}(X_i) - H_i^\top\theta\}=0$, for $\hat\theta$.
\STATE Initialise $\theta^{(0)}$;\ \ $t \leftarrow 0$.
\REPEAT
    \FOR[cross-fitted baseline on transformed outcome $\tilde Y_i(\theta)=Y_i-H_i^\top\theta$]{$k = 1, \dots, K$}
        \STATE Fit $\hat g_{\theta^{(t)}}^{(-k)} \leftarrow \arg\min_{g \in \mathcal{G}} \sum_{i \in I_k^c} \bigl[Y_i - H_i^\top\theta^{(t)} - g(X_i)\bigr]^2$
        \STATE Set $\hat g_{\theta^{(t)}}^{(-)}(X_i) \leftarrow \hat g_{\theta^{(t)}}^{(-k)}(X_i)$ for all $i \in I_k$
    \ENDFOR
    \STATE Structural update: $\theta^{(t+1)} \leftarrow \bigl(\sum_i \hat A_i H_i^\top\bigr)^{-1} \sum_i \hat A_i \bigl[Y_i - \hat g_{\theta^{(t)}}^{(-)}(X_i)\bigr]$ \COMMENT{Eq.~\eqref{eq:theta-hat}}
    \STATE $t \leftarrow t + 1$
\UNTIL{$\|\theta^{(t)} - \theta^{(t-1)}\| < \varepsilon$}
\STATE \quad\COMMENT{Linear baseline $g=\beta^\top b(X)$: the system is linear and converges in one iteration;}
\STATE Set $\hat\theta \leftarrow \theta^{(t)}$.
\medskip
\STATE \textbf{--- Inference ---}
\STATE Out-of-fold residuals: $\hat\psi_i = Y_i - \hat g_{\hat\theta}^{(-)}(X_i) - H_i^\top \hat\theta$
\STATE $\hat G_\theta = n^{-1}\sum_i \hat A_i H_i^\top$, \quad $\hat S = n^{-1}\sum_i \hat\psi_i^2\, \hat A_i \hat A_i^\top$
\STATE $\widehat{\mathrm{Var}}(\hat\theta) = n^{-1}\, \hat G_\theta^{-1}\, \hat S\, \hat G_\theta^{-\top}$ \COMMENT{Eq.~\eqref{eq:sandwich}}
\ENSURE $\hat\theta$, $\widehat{\mathrm{Var}}(\hat\theta)$
\end{algorithmic}
\end{algorithm}

\paragraph{Inference.}
Under the regularity conditions stated in Appendix~\ref{app:inference}, including an $n^{-1/4}$ rate for the cross-fitted mediator CATE and consistency of the cross-fitted baseline, the G-estimator~$\hat\theta$ is asymptotically normal with a sandwich-type variance (see Proposition~\ref{prop:normality} and Appendix~\ref{app:rep-gap} for additional details and derivations).  A consistent estimator of the asymptotic variance is
\begin{equation}\label{eq:sandwich}
\widehat{\Var}(\hat\theta)
\;=\;
\frac{1}{n}\,\hat G_\theta^{-1}\,\hat S\,\hat G_\theta^{-\top},
\end{equation}
where $\hat G_\theta = n^{-1}\sum_i \hat A_i H_i^\top$ and $\hat S = n^{-1}\sum_i \hat\psi_i^2\,\hat A_i\hat A_i^\top$ with out-of-fold residuals $\hat\psi_i = Y_i - \hat g_{\hat\theta}^{(-)}(X_i) - H_i^\top\hat\theta$.

\paragraph{Quality of the estimator.} It is useful to separate three distinct roles that govern the quality of the final estimator:
\begin{enumerate}
\item Identification of~$\theta$ requires treatment ignorability (Assumption~\ref{ass:ignorability}), the NEH assumption (Assumption~\ref{ass:neh}), correct structural specification (Assumption~\ref{ass:model-spec}), and non-degeneracy of the weight matrix (Assumption~\ref{ass:nondegen}).  The centering condition $\E[(T-p)\mid X]=0$ follows from Assumption~\ref{ass:ignorability} under randomisation.
\item Estimation quality of $\tau_M(X)$ determines how closely the plug-in weight $\hat W(X)$ approximates the oracle optimal weight $W^*(X)$.  This is governed by the CATE estimation error, which depends on the learner architecture, the sample size, and the complexity of the true $\tau_M$ \citep{curth2021nonparametric}.
\item Instrument strength determines how much precision the weight matrix delivers, even at the oracle.  The centred weights $(T-p)W(X)$ function as instruments for the structural parameters.  If the underlying DGP has weak treatment-by-covariate interaction on the mediator, i.e., $\Var(\tau_M(X))$ is small, the instruments are weak, and $\hat\theta_2$ will be imprecise regardless of how well $\hat\tau_M$ approximates the truth.
\end{enumerate}
\section{Simulation Study}\label{sec:simulation}

We construct a DGP that mimics typical features of modern and challenging psychological and biomedical data \citep{Micceri1989, Blanca2013, BonoBlancaArnauGomezBenito2017}: non-Gaussian covariates, complex nonlinear effects, and unmeasured mediator--outcome confounding \citep{RudolphGoinPaksarianCrowderMerikangasStuart2019}.
We observe i.i.d.\ tuples $(X_i, T_i, M_i, Y_i)$, $i=1,\dots,n$: covariates
$X\in\R^{25}$, a randomized binary treatment $T\in\{0,1\}$, a continuous mediator
$M$, and a continuous outcome $Y$. Each of the 7 scenarios that we study is derived by modifying one characteristic of the following base model:
\begin{align}
M_i &= \tau(X_i)\,T_i \;+\; \mu_0(X_i) \;+\; K_i \;+\; \varepsilon_{M,i}, \label{eq:M}\\
Y_i &= g(X_i) \;+\; \theta_1 T_i \;+\; \theta_2 M_i \;+\; \lambda_K K_i \;+\; \varepsilon_{Y,i}. \label{eq:Y}
\end{align}
Where $\tau(X)$ is the heterogeneous effect of $T$ on $M$ (the mediator
CATE, target of Stage~1); $\mu_0(X)$ a nonlinear baseline mediator level; $K$ an
unmeasured confounder entering both structural equations (loading $\lambda_K$);
$g(X)=\sum_{j=0}^{6}[0.4X_j+0.2X_j^2\mathbf1\{j\text{ even}\}]$ a nonlinear
nuisance. Parameters $\theta_1=1$ (direct effect), $\theta_2=0.5$ (the mediation
coefficient of interest), $\lambda_K=1$. Treatment is randomized, $T\indep X$,
with known $e(X)\equiv\Pr(T{=}1\mid X)=0.5$. The total effect of $T$ on $Y$
decomposes as $\mathrm{TE}(x)=\theta_1+\theta_2\,\tau(x)$.
Below, we introduce the base generating process for each component, the specific modifications will be made explicit in each of the respective scenarios.

\paragraph{Covariates.}
  Each $X\in\R^{25}$ is generated via a Gaussian copula with Gaussian mixture model (GMM) marginals. In particular we
  draw $Z_i\sim\mathcal{N}_{25}(0,\Sigma)$ with $\Sigma_{jl}=\rho^{|j-l|}$ ($\rho=0.5$), transform to uniform margins $U_{ij}=\Phi(Z_{ij})$, and apply the inverse GMM CDF componentwise to obtain $X_{ij}=F_j^{-1}(U_{ij})$.
  Each marginal $F_j$ is a mixture of $n_j\in\{5,\ldots,10\}$ Gaussian components (independently configured per feature, and number of components sampled uniformly), with component means drawn from $\mathrm{Uniform}(-2,2)$, standard deviations from $\mathrm{Uniform}(0.5,1.5)$, and weights
  proportional to $\mathrm{Uniform}(0.5,2)$ draws; the resulting distribution is min-max scaled to the interval $[0,5]$.
  This yields non-Gaussian, skewed, and multimodal marginals linked by moderate AR(1) copula dependence.

  The 25 features are partitioned into four blocks following the benchmark design of \citet{curth2021nonparametric} as depicted in Table~\ref{tab:sim_cov}:

\begin{table}[htbp]
\centering
\caption{Covariates sizes and roles}
\label{tab:covariates-sizes}
\begin{tabular}{llll}
  \toprule
  Block & Indices & Size & Role\\
  \midrule
  $X_{C,1}$ & 0--4   & 5  & Shared score $s_1$ (enters both $\tau$ and $\mu_0$)\\
  $X_{C,2}$ & 5--9   & 5  & Shared score $s_2$ (enters both $\tau$ and $\mu_0$)\\
  $X_D$     & 15--17 & 3  & Outcome-specific (enters $\mu_0$ via $r_d$)\\
  $X_\emptyset$ & 10--14,\,18--24 & 12 & Noise\\
  \bottomrule
\end{tabular}
\label{tab:sim_cov}
\end{table}

\subsection{Latent-score construction}
\label{subsec:tau-functions}

The treatment effect and baseline outcome are built from three fixed linear projections of the standardized covariates:
$z = (X - \bar X_{\mathrm{}})/s_{\mathrm{}}$:

\begin{align}
s_1 &= \mathbf{c}_1^\top\,z_{0:4}, \quad \mathbf{c}_1 \propto (1.00,\,-0.80,\,0.70,\,0.90,\,-0.70), \label{eq:s1}\\
s_2 &= \mathbf{c}_2^\top\,z_{5:9}, \quad \mathbf{c}_2 \propto (0.90,\,1.00,\,-0.90,\,0.75,\,-0.80), \label{eq:s2}\\
r_d &= \mathbf{c}_d^\top\,z_{15:17}, \quad \mathbf{c}_d \propto (0.90,\,-0.70,\,0.60), \label{eq:rd}
\end{align}

The scores $s_1$ and $s_2$ are shared entering both $\tau(X)$ and $\mu_0(X)$ and provide the joint latent structure.  The treatment-effect function and baseline outcome are:
\begin{align}
\tau(X) &\;=\; 0.45\,\tanh(s_1 s_2)
         \;+\; 0.35\,\sin(s_1+s_2)\,\sigma(s_1-s_2)
         \;+\; 0.20\,\sin(2s_1)\cos(2s_2), \label{eq:tau-honest}\\
\mu_0(X) &\;=\; 0.50\,\mathrm{softplus}(0.8\,s_1 - 0.4\,s_2) \;+\; 0.30\,\tanh(r_d) \;+\; 0.20\,(s_2^2 - 1), \label{eq:mu0-honest}
\end{align}
where $\sigma(\cdot)$ is the logistic function and $\mathrm{softplus}(x)=\log(1+e^x)$.

\paragraph{Unmeasured confounder.}
The unmeasured confounder shares the same distributional family as the observed confounders. In particular we construct the unmeasured confounder~$K$ as a GMM draw mixed with Gaussian noise:
\begin{equation}\label{eq:K-construction}
K \;=\; 0.70\,K_{\mathrm{gmm}} \;+\; 0.30\,\mathcal{N}(0,1),
\end{equation}
where $K_{\mathrm{gmm}}$ is drawn from a three-component Gaussian mixture with means $(-1.2,\,0.3,\,1.5)$, standard deviations $(0.8,\,0.6,\,0.9)$, and mixture weights $(0.40,\,0.35,\,0.25)$.

\paragraph{Variance calibration and scenarios.}
The components $(\tau\!\cdot\!T,\,\mu_0,\,K,\,\varepsilon_M)$ are independently rescaled to achieve target variance shares.

\subsection{The seven scenarios}\label{sec:scenarios}
We design the simulation study to investigate two orthogonal dimensions: identification and efficiency. Specifically, we examine the efficiency of the methods by varying the degree of nonlinearity and the magnitude of $\tau$, while holding the identification structure fixed; this is implemented in scenarios A, B, E, and Wtau. Conversely, we assess the impact of identification by varying the identification conditions while keeping the baseline surface characteristics fixed in scenarios C and D.
\begin{table}[htbp]
\centering
\caption{Simulation scenarios: variance shares by component, NEH status, and predicted bias of $\hat{\theta}_2$.}
\label{tab:simulation-scenarios}
\small
\begin{tabular}{@{}lll l@{}}
\toprule
Scenario & Shares $(\tau,\mu_0,K,\varepsilon_M)$ & \NEH & Pred.\ bias $\hat\theta_2$\\
\midrule
A      & $(0.25,0.33,0.20,0.22)$ & holds    & $0$\\
B      & $(0.40,0.28,0.18,0.14)$ & holds    & $0$\\
C      & $(0.25,0.33,0.20,0.22)$ & holds    & $0$\\
D      & $(0.25,0.33,0.20,0.22)$ & violated & $\gamma\kappa=0.16$\\
Dsmall & $(0.25,0.33,0.20,0.22)$ & violated & $\gamma\kappa=0.04$\\
E      & $(0.25,0.33,0.20,0.22)$ & holds    & $0$\\
Wtau   & $(0.08,0.50,0.20,0.22)$ & holds    & $0$\\
\bottomrule
\end{tabular}
\end{table}

Throughout, $\varepsilon_{M,i}\sim\mathcal N(0,1)$,
$\varepsilon_{Y,i}\sim\mathcal N(0,\sigma_{\varepsilon_Y}^2)$,
$\bar K_i=a_K K_i$, and $(a_\tau,a_D,a_K,a_\varepsilon)$ are the
calibration scales for the cell's shares.

\paragraph{A: realistic heterogeneity ($\tau$-share $25\%$).}
\begin{equation}\label{eq:A}
{
\begin{aligned}
M_i &= a_\tau\,\tau(X_i)\,T_i + a_D\,\mu_0(X_i) + \bar K_i + a_\varepsilon\,\varepsilon_{M,i},\\
Y_i &= g(X_i) + \theta_1 T_i + \theta_2 M_i + \lambda_K\,\bar K_i + \varepsilon_{Y,i},
\qquad \sigma_{\varepsilon_Y}^2=\sigma_0^2-\lambda_K^2\widehat\Var(\bar K).
\end{aligned}\;}
\end{equation}
This is the reference scenario NEH holds; predicted bias $0$.

\paragraph{B: strong heterogeneity ($\tau$-share $40\%$).}
This case has exactly the same structure as in Eq.~\eqref{eq:A}, but assigns a much larger weight to the $\tau$ component, namely $(0.40,0.28,0.18,0.14)$, thus corresponding to an “ideal instrument” setting. Under this specification, NEH\ is satisfied and the predicted bias equals $0$.

\paragraph{C: rank preservation fails, \NEH\ holds.}
The mediator is specified as in Eq.~\eqref{eq:A}. The outcome slope with respect to \(M\) becomes idiosyncratic through the inclusion of the term \(\xi_i\):
\begin{equation}\label{eq:C}
{
\begin{aligned}
M_i &= a_\tau\,\tau(X_i)\,T_i + a_D\,\mu_0(X_i) + \bar K_i + a_\varepsilon\,\varepsilon_{M,i},\\
Y_i &= g(X_i) + \theta_1 T_i + (\theta_2+\xi_i)\,M_i + \lambda_K\,\bar K_i + \varepsilon_{Y,i},
\quad \xi_i\sim\mathcal N(0,s_\xi^2)\ \text{i.i.d.},
\end{aligned}\;}
\end{equation}
The individual rank preservation assumption fails,
but NEH holds. Predicted bias $0$.

\paragraph{D: essential heterogeneity, \NEH\ violated ($\gamma=\kappa=0.4$).}
A latent $L_i\sim\mathcal N(0,1)$, independent of $(X,T,K,\varepsilon_M,\varepsilon_Y)$,
enters both equations:
\begin{equation}\label{eq:D}
{
\begin{aligned}
M_i &= a_\tau\,\tau(X_i)\,(1+\kappa L_i)\,T_i + a_D\,\mu_0(X_i) + \bar K_i + a_\varepsilon\,\varepsilon_{M,i},\\
Y_i &= g(X_i) + \theta_1 T_i + \theta_2 M_i + \gamma\,L_i M_i + \lambda_K\,\bar K_i + \varepsilon_{Y,i},
\end{aligned}}
\end{equation}
The same $L$ modulates the mediator's treatment response ($\kappa L$) and the outcome's mediator slope ($\gamma L$). \NEH\ violated; predicted bias $\gamma\kappa=0.16$.

\paragraph{Dsmall: half amplitude ($\gamma=\kappa=0.2$).}
Eq.~\eqref{eq:D} with $\gamma=\kappa=0.2$, predicted bias $0.04$. 

\paragraph{E: CATE (mostly) linear, \NEH\ holds.}
The response surface is replaced, by an orthogonalized combination of linear and nonlinear components.
$\tau_E$; $\mu_0,K$, shares, noise are those of scenario A:
\begin{equation}\label{eq:E}
{
\begin{aligned}
M_i &= a_\tau\,\tau_E(X_i)\,T_i + a_D\,\mu_0(X_i) + \bar K_i + a_\varepsilon\,\varepsilon_{M,i},\\
Y_i &= g(X_i) + \theta_1 T_i + \theta_2 M_i + \lambda_K\,\bar K_i + \varepsilon_{Y,i},
\quad \sigma_{\varepsilon_Y}^2=\sigma_0^2-\lambda_K^2\widehat\Var(\bar K),
\end{aligned}}
\end{equation}
\begin{equation}\label{eq:E-tau}
\tau_E(X)=w_{\mathrm{nl}}\,\frac{z_\perp(X)}{\mathrm{sd}(z_\perp)}
         +w_{\mathrm{lin}}\,\frac{z_{\mathrm{lin}}(X)}{\mathrm{sd}(z_{\mathrm{lin}})},
\quad z_{\mathrm{lin}}=0.5 s_1+0.5 s_2,\;\;
z_\perp=z_{\mathrm{nl}}-\beta\,z_{\mathrm{lin}},
\end{equation}

$z_{\mathrm{nl}}$ is built as in Eq.~\eqref{eq:tau-honest},
$\beta=\Cov(z_{\mathrm{nl}},z_{\mathrm{lin}})/\Var(z_{\mathrm{lin}})$ and the
reference SDs computed once on a large population draw (DGP identical across $n$),
$(w_{\mathrm{nl}},w_{\mathrm{lin}})=(0.30,0.90)$. Realized full-linear $R^2$ of
$\tau_E$ averages $0.89$ (vs $0.19$ in scenario A). 

\paragraph{Wtau: weak instrument, \NEH\ holds.}
Eq.~\eqref{eq:A}, shares $(0.08,0.50,0.20,0.22)$: the $\tau$-share is
$8\%$, It corresponds to a weak-instrument regime, translating in a weak signal for Stage-1. NEH holds; predicted bias $0$.

\subsection{Stage~1 methods}
\label{subsec:stage1-methods}

All Stage~1 methods use 5-fold cross-fitting; implementation details and hyperparameters are given in the Estimator Hyperparameters appendix~\ref{app:hyper}.
We compare\footnote{ \textbf{Thin Plate Splines} were considered as a candidate middle-tier learner, motivated by their use in semiparametric mediator modeling \citep{Brandt2020}. However, in preliminary runs at $d=25$ TPS was numerically unstable despite regularisation, so we excluded it and do not report TPS results \cite{Kalogridis2023}}:
\begin{itemize}
\item \textbf{TARNet}: SNet-1 architecture \citep{curth2021nonparametric} with shared representation layers and treatment-specific heads.  The CATE is computed directly as $\hat\tau(X) = h_1(\Phi(X)) - h_0(\Phi(X))$.

\item \textbf{TNet (no shared representation)}: the same SNet-1 architecture as TARNet \citep{curth2021nonparametric} but without the shared representation layers, so each treatment arm is fitted separately. The CATE is computed as $\hat\tau(X) = h_1(X) - h_0(X)$. This learner serves as the ablation that isolates the value of the shared representation.
\item \textbf{Ridge (T-learner)}: Separate Ridge regressions for $T=0$ and $T=1$, with $\hat\tau(X) = \hat\mu_1(X) - \hat\mu_0(X)$.  Near-optimal when $\tau$ is linear but unable to capture interactions or nonlinearities.

\item \textbf{Random Forest (T-learner)} \cite{Breiman2001}: Separate random forests for $T=0$ and $T=1$.  Hyperparameters (number of estimators, maximum depth, maximum features fraction) are selected by 3-fold cross-validation with 6 random configurations per fold. Captures nonlinearity and feature interactions via tree splits but lacks TARNet's shared-representation inductive bias.
\item \textbf{Baron--Kenny estimator.}
The naive baseline is the \citet{BaronKenny1986} two-regression estimator. Because it omits the unmeasured confounder $K$ from the outcome regression, it is biased for $\theta_2$ by a positive amount that does not vanish with $n$. Empirically (see Table~\ref{tab:results}), this bias is $+0.27$ to $+0.39$ on the true $\theta_2=0.5$, with $0\%$ coverage in every cell, irrespective of sample size or signal strength.

\end{itemize}

 \subsection{Monte-Carlo design and evaluation metrics}\label{sec:metrics}

Each scenario is replicated over $R = 200$ random seeds, for sample sizes
$n \in \{500,\,1000,\,2000,\,5000,\,10000\}$. The low-sample-size settings
$n \in \{500,\,1000\}$ are included to assess the behaviour of the methods in small-sample regimes.\footnote{However, for applied analyses, we recommend avoiding sample sizes below 1000.}
The structural parameters $\theta_1=1.0$ and $\theta_2=0.5$ are fixed across replications. Each replication produces a single point estimate $\hat\theta_2^{(r)}$, the sandwich estimator $\mathrm{SE}_r$ from Eq.~\eqref{eq:sandwich}, and a confidence interval $\mathrm{CI}_r=\hat\theta_2^{(r)}\pm1.96\,\mathrm{SE}_r$.
The reported per-cell, per-method metrics are depicted in Table~\ref{tab:simulation-metrics}:
\begin{table}[htbp]
\centering
\caption{Performance metrics used in the simulation study, their definitions, and target interpretation.}
\label{tab:simulation-metrics}
\small
\begin{tabular}{@{}lll@{}}
\toprule
Metric & Definition & Target / reading\\
\midrule
Median bias & $\Med_r\big(\hat\theta_2^{(r)}-\theta_2\big)$ & $0$ if consistent\\
Empirical SD & $\mathrm{sd}_r\big(\hat\theta_2^{(r)}\big)$ & \\
Median SE & $\Med_r(\mathrm{SE}_r)$ & sandwich estimate of the above\\
SE calibration & $\overline{\mathrm{SE}}/\,\mathrm{sd}_r(\hat\theta_2^{(r)})$ & $\approx1$ $\Leftrightarrow$ valid sandwich\\
Coverage & $\frac1R\sum_r\mathbf 1\{\theta_2\in\mathrm{CI}_r\}$ & nominal $0.95$ (larger is better)\\
Power / rejection & $\frac1R\sum_r\mathbf 1\{0\notin\mathrm{CI}_r\}$ & power vs $H_0{:}\theta_2{=}0$\\
Stage-1 NRMSE & $\overline{\|\hat\tau-\tau\|_2/\|\tau-\bar\tau\|_2}$ & CATE quality; $0$ perfect, $1$ no better than the mean; smaller is better\\
Efficiency ratio & $\Med_r(\mathrm{SE}_r^{\text{method}}/\mathrm{SE}_r^{\text{TARNet}})$ & $>1\Leftrightarrow$ less efficient than TARNet\\
\bottomrule
\end{tabular}
\end{table}

A replication's Stage~2 is skipped if
the Stage-1 NRMSE $>1.3$, or $\mathrm{sd}(\hat\tau)<10^{-4}$. A replication is likewise excluded if Stage~2 does not
converge, any $\hat\theta$ or SE is non-finite. Table~\ref{tab:skip} reports skip counts
at $n=500$; at $n\ge1000$ skips fall below $2\%$ except for Wtau, and are zero at
$n\ge5000$.

\begin{table}[ht]\centering\small
\caption{Stage-2 skip counts at $n=500$ (out of $R=200$ seeds per learner). Skips
track Stage-1 difficulty: none for the strong instrument (B) or the near-linear
surface (E), heavy for the latent-noise channel (D) and the weak instrument
(Wtau), where every TARNet seed is skipped. Skips fall below $2\%$ at $n\ge1000$
(except Wtau) and vanish at $n\ge5000$.}
\label{tab:skip}
\begin{tabular}{@{}lcccc@{}}
\toprule
Scenario & Ridge & RF & TNet & TARNet\\
\midrule
A & 17 & 15 & 80 & 35\\
B & 10 & 4 & 2 & 0\\
C & 18 & 16 & 79 & 32\\
D & 26 & 27 & 138 & 104\\
Dsmall & 17 & 23 & 94 & 46\\
E & 0 & 6 & 3 & 1\\
Wtau & 102 & 66 & 200 & 200\\
\bottomrule
\end{tabular}
\end{table}

\subsection{Results}\label{sec:part2-results}

\paragraph{Stage-1 CATE quality.}
Stage-1 quality is measured as the learner’s normalized RMSE for $\hat\tau$ (see Figure~\ref{fig:nrmse}).  In the nonlinear surface scenarios (A, B, C, D, Dsmall), TARNet’s NRMSE decreases as $n$ grows, whereas Ridge levels off above $0.77$ (reflecting the irreducible error from linear misspecification). In scenario E (linear surface), Ridge’s NRMSE falls to $0.36$, while RF, which lacks a linear inductive bias, flattens at $0.61$. In the Wtau (weak instrument) scenario, the NRMSE increases for all learners, with TARNet exceeding the $1.3$ threshold at $n{=}500$.

\begin{figure}[ht]\centering
\includegraphics[width=\textwidth]{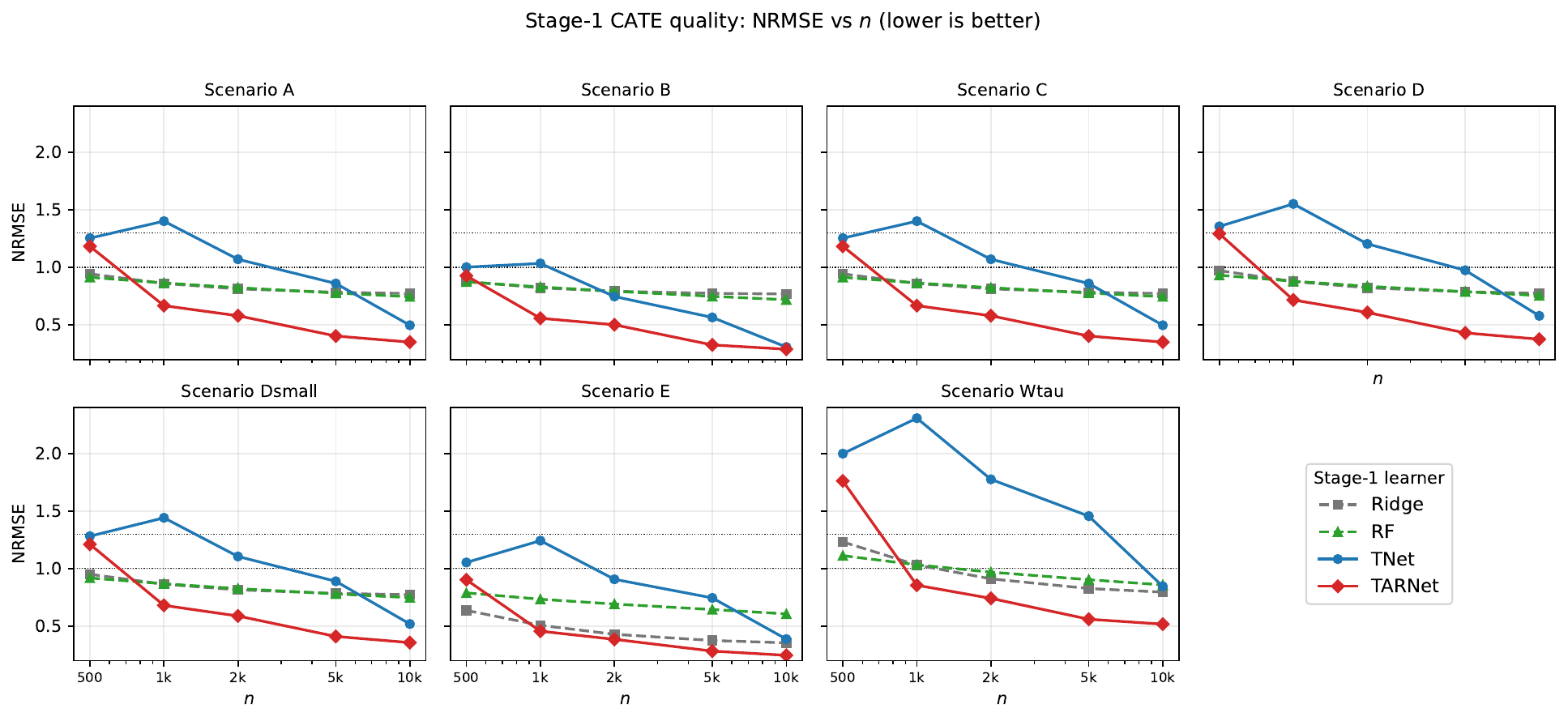}
\caption{Stage-1 NRMSE of $\hat\tau$ versus $n$ (log scale), one panel per
scenario, four learners. TARNet falls steadily on the nonlinear surfaces while
Ridge plateaus above $0.77$ (irreducible linear misspecification); on the
near-linear surface~E Ridge drops to $0.36$ while RF stalls at $0.61$, and under
the weak instrument (Wtau) every learner is inflated, TARNet exceeding the $1.3$
skip threshold at $n{=}500$.}
\label{fig:nrmse}
\end{figure}

The correlation $\Corr(\hat\tau,\tau)$
(Figure~\ref{fig:corr}) between true and estimated tau provides additional information about the performance of the learners.
Across all scenarios from A through Dsmall, a very similar pattern can be observed. TARNet yields the highest correlations across all sample sizes $n$, with correlations close to or above $0.8$ for $n \geq 1000$. For $n = 500$, all learners perform similarly in these scenarios, with relatively low correlations ranging between $0.4$ and $0.6$.

In scenario E, corresponding to the linear setting, both TARNet and Ridge achieve similarly high correlations above $0.9$ for $n \geq 1000$. In scenario Wtau, corresponding to the weak-instrument setting, a similar ranking of the learners is observed, although correlations are generally lower. In this case, TARNet reaches correlations close to $0.8$ only for $n \geq 5000$.

\begin{figure}[ht]\centering
\includegraphics[width=\textwidth]{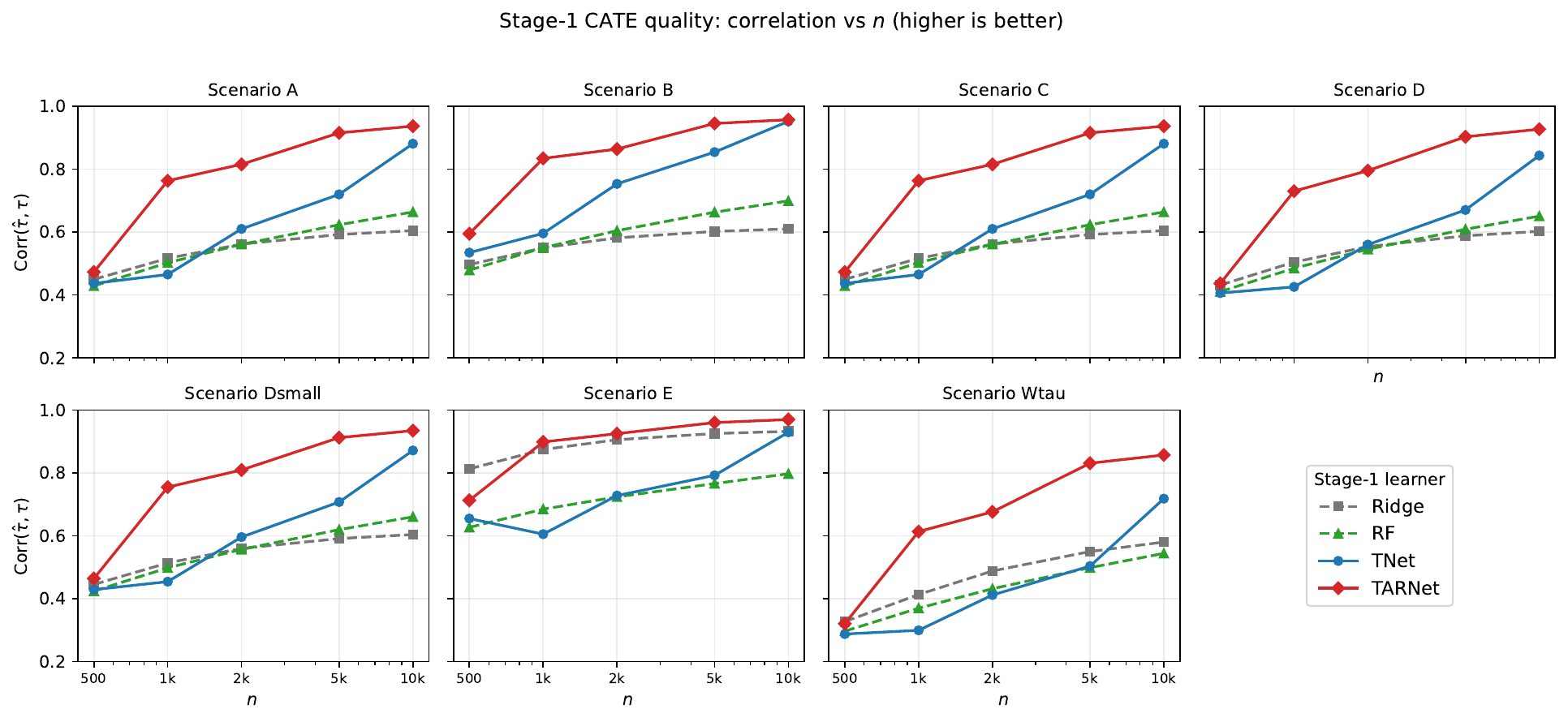}
\caption{Stage-1 $\Corr(\hat\tau,\tau)$ vs.\ $n$ by learner and scenario. TARNet (shared) leads at every $n$; TNet
(not-shared) starts low and converges toward TARNet by $n{=}10{,}000$ while
overtaking the RF around $n{=}2000$; on the near-linear surface E, Ridge is
strong throughout. Both the deep-learning advantage and the shared-representation
gap (TARNet $-$ TNet) are nonlinearity-specific.}
\label{fig:corr}
\end{figure}

\paragraph{Stage-2 estimation.}
Table~\ref{tab:results} summarizes the main patterns. In the settings where NEH holds (A, B, C,
E, Wtau), all learners are unbiased (albeit the presence of a minor empirical bias is to be expected) and achieve nominal coverage for $n\ge1000$, whereas BK exhibits
a positive bias of $+0.27$ to $+0.33$ and $0\%$ coverage at all sample sizes. When NEH is violated (D,
Dsmall), each learner converges to the expected $\gamma\kappa$ bias, and coverage
collapses as $n$ increases. This breakdown is most rapid for TARNet, whose narrower intervals concentrate
around the wrong parameter. The efficiency ratios indicate that TARNet achieves the smallest
SE in every nonlinear design (R/T, RF/T $\approx1.4$--$1.6$ for $n\ge1000$) and remains at
roughly $\approx1.0$ on the linear setting E. The TT/T column isolates the penalty of removing
the shared representation: TNet lags TARNet by $30$--$40\%$ in SE at moderate $n$
($1.3$--$1.4$ at $n{=}1000$--$2000$ in A, B, C, Dsmall), but this difference shrinks to
about $\approx1.0$ by $n{=}10{,}000$, indicating a finite-sample benefit of the shared representation rather than an asymptotic
one.

\begin{table}[!ht]\centering\footnotesize
\caption{Stage-2 results: median bias $\hat\theta_2-0.5$ with coverage rates in brakets for the four
learners, BK median bias, and median per-seed SE ratios R/T (Ridge$\div$TARNet),
RF/T (RF$\div$TARNet) and TT/T (TNet$\div$TARNet, i.e.\ not-shared $\div$ shared
neural; ratio of median SEs since TNet is a separate pool). ``E'' is the near-linear
scenario; Wtau at $n{=}500$ skips both neural learners in all seeds. Dashes (---) mark skipped cells (Table~\ref{tab:skip}); BK $=$ Baron--Kenny.}
\label{tab:results}
\begin{tabular}{@{}cl cccc c ccc@{}}
\toprule
& $n$ & Ridge $b$(cov) & RF $b$(cov) & TNet $b$(cov) & TARNet $b$(cov) & BK & R/T & RF/T & TT/T\\
\midrule
\multirow{5}{*}{A} & 500 & $+.029$ (.96) & $+.027$ (.95) & $-.029$ (.93) & $-.002$ (.93) & $+.299$ & 1.03 & 1.16 & 1.00\\
 & 1000 & $-.009$ (.96) & $-.003$ (.98) & $-.013$ (1.00) & $-.011$ (.95) & $+.302$ & 1.45 & 1.53 & 1.38\\
 & 2000 & $+.008$ (.97) & $+.017$ (.95) & $+.004$ (.96) & $+.008$ (.94) & $+.303$ & 1.41 & 1.43 & 1.33\\
 & 5000 & $-.008$ (.94) & $+.000$ (.94) & $+.001$ (.94) & $+.004$ (.96) & $+.303$ & 1.51 & 1.48 & 1.28\\
 & 10000 & $-.003$ (.97) & $-.003$ (.95) & $-.001$ (.95) & $-.001$ (.96) & $+.302$ & 1.53 & 1.42 & 1.06\\
\midrule
\multirow{5}{*}{B} & 500 & $+.014$ (.95) & $+.001$ (.94) & $-.001$ (.92) & $+.004$ (.93) & $+.271$ & 1.14 & 1.27 & 1.11\\
 & 1000 & $+.000$ (.95) & $-.004$ (.96) & $-.004$ (.94) & $-.003$ (.94) & $+.272$ & 1.48 & 1.51 & 1.40\\
 & 2000 & $+.004$ (.97) & $+.006$ (.95) & $-.000$ (.97) & $+.001$ (.94) & $+.272$ & 1.44 & 1.43 & 1.15\\
 & 5000 & $-.004$ (.93) & $-.001$ (.92) & $+.000$ (.94) & $+.002$ (.95) & $+.271$ & 1.53 & 1.44 & 1.11\\
 & 10000 & $-.003$ (.99) & $-.002$ (.95) & $-.000$ (.95) & $-.001$ (.96) & $+.273$ & 1.55 & 1.38 & 1.01\\
\midrule
\multirow{5}{*}{C} & 500 & $+.031$ (.96) & $+.011$ (.95) & $-.025$ (.95) & $+.013$ (.93) & $+.301$ & 1.04 & 1.15 & 0.98\\
 & 1000 & $+.001$ (.97) & $+.014$ (.98) & $+.001$ (1.00) & $-.009$ (.97) & $+.304$ & 1.41 & 1.49 & 1.37\\
 & 2000 & $+.010$ (.97) & $+.013$ (.97) & $+.009$ (.96) & $+.007$ (.95) & $+.299$ & 1.40 & 1.41 & 1.32\\
 & 5000 & $-.002$ (.95) & $-.003$ (.94) & $+.003$ (.96) & $+.003$ (.96) & $+.304$ & 1.50 & 1.45 & 1.26\\
 & 10000 & $-.006$ (.98) & $-.002$ (.96) & $-.000$ (.96) & $+.000$ (.95) & $+.301$ & 1.52 & 1.41 & 1.06\\
\midrule
\multirow{5}{*}{D} & 500 & $+.179$ (.87) & $+.174$ (.87) & $+.085$ (.92) & $+.128$ (.91) & $+.382$ & 0.97 & 1.04 & 1.07\\
 & 1000 & $+.140$ (.81) & $+.154$ (.82) & --- (---) & $+.151$ (.70) & $+.385$ & 1.36 & 1.46 & ---\\
 & 2000 & $+.160$ (.57) & $+.161$ (.54) & $+.163$ (.56) & $+.150$ (.34) & $+.386$ & 1.34 & 1.39 & 1.37\\
 & 5000 & $+.157$ (.28) & $+.156$ (.24) & $+.157$ (.14) & $+.162$ (.02) & $+.386$ & 1.47 & 1.45 & 1.32\\
 & 10000 & $+.153$ (.09) & $+.162$ (.00) & $+.158$ (.00) & $+.159$ (.00) & $+.384$ & 1.51 & 1.40 & 1.09\\
\midrule
\multirow{5}{*}{Dsmall} & 500 & $+.054$ (.95) & $+.057$ (.93) & $-.013$ (.93) & $+.027$ (.94) & $+.324$ & 1.02 & 1.07 & 1.03\\
 & 1000 & $+.029$ (.96) & $+.034$ (.94) & $+.062$ (1.00) & $+.020$ (.95) & $+.327$ & 1.41 & 1.50 & 1.42\\
 & 2000 & $+.045$ (.91) & $+.052$ (.92) & $+.039$ (.90) & $+.041$ (.87) & $+.325$ & 1.39 & 1.42 & 1.36\\
 & 5000 & $+.037$ (.84) & $+.040$ (.83) & $+.040$ (.82) & $+.044$ (.77) & $+.326$ & 1.51 & 1.46 & 1.29\\
 & 10000 & $+.034$ (.82) & $+.038$ (.76) & $+.040$ (.65) & $+.040$ (.60) & $+.325$ & 1.52 & 1.42 & 1.07\\
\midrule
\multirow{5}{*}{E} & 500 & $-.004$ (.95) & $+.008$ (.97) & $-.027$ (.94) & $-.017$ (.95) & $+.314$ & 0.89 & 1.14 & 1.06\\
 & 1000 & $+.001$ (.96) & $-.008$ (.96) & $+.006$ (.94) & $-.002$ (.96) & $+.318$ & 1.03 & 1.32 & 1.43\\
 & 2000 & $+.003$ (.94) & $+.008$ (.96) & $+.003$ (.95) & $+.002$ (.95) & $+.318$ & 1.02 & 1.28 & 1.29\\
 & 5000 & $-.004$ (.93) & $-.003$ (.94) & $-.002$ (.94) & $-.001$ (.93) & $+.316$ & 1.04 & 1.26 & 1.21\\
 & 10000 & $-.003$ (.97) & $-.002$ (.95) & $-.003$ (.95) & $-.001$ (.96) & $+.316$ & 1.04 & 1.22 & 1.04\\
\midrule
\multirow{5}{*}{Wtau} & 500 & $+.078$ (.93) & $+.148$ (.93) & --- (---) & --- (---) & $+.329$ & --- & --- & ---\\
 & 1000 & $-.017$ (.95) & $+.039$ (.96) & --- (---) & $-.055$ (.98) & $+.333$ & 1.39 & 1.52 & ---\\
 & 2000 & $+.009$ (.97) & $+.016$ (.96) & --- (---) & $+.004$ (.95) & $+.330$ & 1.35 & 1.56 & ---\\
 & 5000 & $-.015$ (.94) & $-.005$ (.94) & $+.123$ (1.00) & $+.007$ (.94) & $+.333$ & 1.42 & 1.68 & 1.71\\
 & 10000 & $-.008$ (.98) & $-.003$ (.96) & $-.007$ (.94) & $-.001$ (.95) & $+.332$ & 1.41 & 1.59 & 1.19\\
\bottomrule
\end{tabular}
\end{table}

 \section{Discussion}
  \label{sec:discussion}

  \subsection{Summary of findings}

  Our results trace a two-stage efficiency chain,
  \[
  \text{Better } \hat\tau_M \text{ (Stage~1)}
  \;\Longrightarrow\;
  \text{More informative weight } \hat A
  \;\Longrightarrow\;
  \text{More efficient } \hat\theta_2 \text{ (Stage~2)},
  \]
  Across the nonlinear
  scenarios the TARNet weight delivers Stage~2 standard errors about $1.5\times$
  smaller than a linear T-learner (median Ridge/TARNet ratio $1.45$--$1.51$ at
  $n\ge2{,}000$). 

Identification behaves exactly as the \NEH\ theory predicts, and independently of the Stage-1 learner. Where \NEH\ holds, every G-estimator is approximately unbiased with near-nominal coverage; where it fails (Scenario~D) all learners converge to the same predicted $\gamma\kappa$ offset, and coverage degrades predictably as the sample grows and the interval tightens
around the biased target. Scenario~C confirms that the method requires only \NEH\ and not rank preservation: its results are not significantly different from the baseline scenario. The naive Baron--Kenny estimator, carries a
persistent bias of $+0.27$ to $+0.39$ on $\theta_2=0.5$ with $0\%$ coverage in every cell, regardless of sample size or signal strength.

  \subsection{Shared representation versus separate arms}
Given the design of the simulation we can separate the two distinct mechanisms that drive the efficiency gains.

First, flexibility by itself is insufficient. Although the Random Forest is fully nonparametric, it only slightly outperforms Ridge, with $\Corr(\hat\tau_M,\tau_M)\approx0.65$--$0.70$ even at $n=10{,}000$, whereas the neural learners reach $0.84$--$0.96$. The core issue is structural: the mediator CATE is defined via smooth functions of dense linear projections ($s_1,s_2$) of the covariates, and an axis-aligned splitting scheme cannot capture these projections. This also provides a mechanistic explanation for how a highly flexible tree ensemble can fail to beat a linear model when the underlying structure is governed by latent
 linear relationships.

Second, the shared representation contributes beyond the neural function class. TNet uses the same neural architecture as TARNet but fits each treatment arm separately. Although it learns the latent projections, because each arm only sees the units in its own treatment group it approaches the shared TARNet only asymptotically. At $n=1{,}000$ it lags substantially ($\Corr\approx0.46$ vs $0.76$ in Scenario~A), but the gap closes by $n=10{,}000$ ($0.88$ vs $0.94$).

\subsection{Limitations and Future Research}
\label{sec:limitations}

While the UNIT framework provides a robust integration of deep representation learning with G-estimation, it is important to acknowledge some limitations and to highlight promising directions for future research.

\paragraph{Estimation of the Nuisance Baseline} 
In our current implementation, the baseline outcome model $g(X)$ is estimated via 5-fold cross-fitted Ridge regression within Stage 2, as prescribed in Section~\ref{subsec:algoverview}. For the linear specifications of $g$ employed in our simulations, the numerical discrepancy between cross-fitted and in-sample estimates was negligible. Consequently, the DML correction \cite{Chernozhukov2018} provided primarily theoretical asymptotic safeguards rather than substantial finite-sample gains in these specific settings. Future research should evaluate the performance of UNIT in contexts where $g(X)$ is highly non-linear,  a detailed discussion of the complications that arise under nonlinear~$g$ is given in Section~\ref{subsec:omega-role}.

\paragraph{Extension to Observational Data} 
The methodology developed in this work is specifically designed for contexts in which the treatment assignment mechanism is randomized (for example, randomized clinical trials or A/B testing environments). Extending this framework to observational study designs would necessitate the simultaneous estimation of the propensity score \(e(X)\) and its explicit incorporation into the construction of the weight function \(\hat{A}_i\).

\paragraph{Empirical Benchmarking}

In contrast to standard CATE estimation, where performance can be assessed using widely adopted benchmarks such as IHDP \citep{Hill2011} or ACIC \citep{HahnDorieMurray2019}, to the best of our knowledge, no public benchmark datasets are currently available for the mediation setting considered here. 
Developing and analyzing experimental datasets to further evaluate the empirical performance of UNIT and competing methods, therefore, constitutes an important direction for future research.

\subsection{Conclusion}
\label{sec:conclusion}

We introduced UNIT, a two-stage estimator that pairs a TARNet-learned mediator CATE with the \citet{ZhengZhou2015} G-estimation of a structural mean model under no essential heterogeneity. Identification of the structural parameters is inherited from the NEH framework and holds regardless of the Stage-1 learner; our contribution is to show, both theoretically and in simulation, that a more accurate first-stage representation yields a more informative plug-in weight and thereby a more efficient structural estimator, in high-dimensional and non linear settings. Across nonlinear mediator-effect surfaces, the shared-representation learner reduces the Stage-2 standard error of the mediation coefficient by roughly one third relative to a linear T-learner, while maintaining low bias and close to nomimal coverage rates as long as NEH assumption holds. For applied mediation analyses with randomised treatment and complex covariate--mediator relationships, UNIT offers a substantive efficiency gain compared to other learners (eg. Ridge, RF or TNet) within a semiparametrically principled framework.

\section*{Data and Code Availability}

All code and frozen simulation data needed to reproduce every table and figure in this paper are available at \url{https://github.com/PsychometricsMZ/UNIT}. Each cell aggregates $R=200$ Monte-Carlo seeds; the seed scheme and all hyperparameters are documented in the Appendix. The software stack is Python (NumPy, scikit-learn, and JAX/CATENets for the TARNet learner).

\section*{Funding}
No funding was received for conducting this study.

\section*{Competing Interests}
The authors declare that they have no competing interests.

\bibliographystyle{apalike}
\bibliography{references}

\appendix

\section{Detailed Derivation: Representation Gap}
\label{app:rep-gap}

We derive the asymptotic variance formula and we show how Theorem~2 of \citet{ZhengZhou2015} specializes to the binary-treatment structural mean model used in our paper, and how the weight quality enters the variance expression.

\subsection{The Zheng--Zhou joint estimating system}

Following \citet[Section~4.5]{ZhengZhou2015}, define the stacked estimating function
\[
G_i(\theta, \beta) \;=\; \begin{pmatrix} G_{1i}(\theta,\beta) \\ G_{2i}(\theta,\beta) \end{pmatrix},
\]
where the $\theta$-block uses the centred weight $\tilde A(T_i,X_i) = A(T_i,X_i) - \E[A(T_i,X_i)\mid X_i]$:
\begin{equation}\label{eq:G1-def}
G_{1i}(\theta,\beta) \;=\; \tilde A(T_i,X_i)^\top\!\left\{Y_i - \sum_{k=1}^K \theta_k\,h_k(T_i,M_i,X_i) - \tilde g(X_i,\beta)\right\},
\end{equation}
and the $\beta$-block uses a separate instrument $B(X_i)$:
\begin{equation}\label{eq:G2-def}
G_{2i}(\theta,\beta) \;=\; B(X_i)^\top\!\left\{Y_i - \sum_{k=1}^K \theta_k\,h_k(T_i,M_i,X_i) - \tilde g(X_i,\beta)\right\}.
\end{equation}

\subsection{Asymptotic normality}

\begin{proposition}[Following Theorem~2 of \citet{ZhengZhou2015}]\label{prop:zz-thm2}
Let $\theta^{(0)}$ denote the true value of~$\theta$ and let $\beta^{(0)}$ be the unique solution of $\E[G_{2i}(\theta^{(0)},\beta)] = 0$.  Under the NEH assumption and standard regularity conditions, the joint estimator $(\hat\theta_n, \hat\beta_n)$ satisfies
\begin{equation}\label{eq:zz-joint-normality}
\sqrt{n}\begin{pmatrix}\hat\theta_n - \theta^{(0)} \\ \hat\beta_n - \beta^{(0)}\end{pmatrix}
\;\xrightarrow{d}\;
\mathcal{N}\!\left(0,\; \mathcal{A}^{-1}\,\mathcal{B}\,\mathcal{A}^{-\top}\right),
\end{equation}
where $\mathcal{A}$ is the Jacobian
\begin{equation}\label{eq:calA}
\mathcal{A} \;=\; \E\!\left[\frac{\partial G_i(\theta^{(0)},\beta^{(0)})}{\partial(\theta^\top, \beta^\top)}\right]
\;=\;
\begin{pmatrix}
\E\!\left[\dfrac{\partial G_{1i}}{\partial\theta}\right] & \E\!\left[\dfrac{\partial G_{1i}}{\partial\beta}\right] \\[10pt]
\E\!\left[\dfrac{\partial G_{2i}}{\partial\theta}\right] & \E\!\left[\dfrac{\partial G_{2i}}{\partial\beta}\right]
\end{pmatrix},
\end{equation}
and $\mathcal{B}$ is the variance of the score:
\begin{equation}\label{eq:calB}
\mathcal{B} \;=\; \E\!\big[G_i(\theta^{(0)},\beta^{(0)})\,G_i(\theta^{(0)},\beta^{(0)})^\top\big].
\end{equation}
The asymptotic covariance of~$\hat\theta_n$ is the $\theta$-block of $\mathcal{A}^{-1}\mathcal{B}\,\mathcal{A}^{-\top}$.
\end{proposition}

\begin{proof}
\textbf{Step~1 (Consistency of~$\hat\theta_n$).}
At the true parameters, $\E[G_{1i}(\theta^{(0)},\beta^{(0)})] = 0$ because $\E[\tilde A(T,X)\mid X] = 0$ by construction and $\E[\tilde Y(\theta^{(0)})\mid T,X] = g(X)$.  The sample average of $G_{1i}$ at $(\hat\theta_n,\hat\beta_n)$ equals zero by definition of the estimator, so by the weak law of large numbers:
\[
0 \;=\; \E[G_1(\hat\theta_n,\hat\beta_n)] + o_p(1).
\]
A first-order Taylor expansion of $\E[G_1(\hat\theta_n,\hat\beta_n)]$ around $(\theta^{(0)},\beta^{(0)})$ gives
\begin{equation}\label{eq:taylor-G1}
0 \;=\; \E[G_1(\theta^{(0)},\beta^{(0)})] \;-\; \E[\tilde A^\top H]\,(\hat\theta_n - \theta^{(0)}) \;+\; \E\!\left[\frac{\partial G_1}{\partial\beta}\right](\hat\beta_n - \beta^{(0)}) \;+\; o_p(1).
\end{equation}

Using the centering property: for any $\theta$ and $\beta$,
\begin{align}
\E\!\left[\frac{\partial G_{1i}}{\partial\beta}\right]
&= -\E\!\left[\tilde A(T_i,X_i)^\top\,\frac{\partial\tilde g(X_i,\beta)}{\partial\beta}\right] \nonumber \\
&= -\E\!\left[\E\!\big[\tilde A(T_i,X_i)^\top\mid X_i\big]\,\frac{\partial\tilde g(X_i,\beta)}{\partial\beta}\right] \;=\; 0, \label{eq:centering-key}
\end{align}
because $\E[\tilde A(T,X)\mid X] = 0$.  This is the working-model property of G-estimation: consistency of~$\hat\theta$ does not depend on correct specification of~$\tilde g$.

Since $\E[G_1(\theta^{(0)},\beta^{(0)})] = 0$ and Eq.~\eqref{eq:centering-key} eliminates the $\hat\beta_n$ term, we obtain
\[
0 \;=\; -\E[\tilde A^\top H]\,(\hat\theta_n - \theta^{(0)}) + o_p(1).
\]
When $\E[\tilde A^\top H]$ is nonsingular (Assumption~\ref{ass:nondegen}), this yields $\hat\theta_n \xrightarrow{p} \theta^{(0)}$.

\medskip
\textbf{Step~2 (Consistency of~$\hat\beta_n$).}
Since $\hat\theta_n \xrightarrow{p} \theta^{(0)}$ and $n^{-1}\sum_i G_{2i}(\hat\theta_n,\hat\beta_n) = 0$, the law of large numbers gives $\E[G_2(\theta^{(0)},\hat\beta_n)] = o_p(1)$.  The uniqueness of~$\beta^{(0)}$ as the solution to $\E[G_2(\theta^{(0)},\beta)] = 0$ implies $\hat\beta_n \xrightarrow{p} \beta^{(0)}$.

\medskip
\textbf{Step~3 (Joint asymptotic normality).}
Both estimators are now consistent.  Applying a first-order Taylor expansion to the full stacked system:
\[
0 \;=\; n^{-1}\sum_{i=1}^n G_i(\hat\theta_n, \hat\beta_n)
\;=\; n^{-1}\sum_{i=1}^n G_i(\theta^{(0)}, \beta^{(0)})
\;+\; \bar\nabla\,\begin{pmatrix}\hat\theta_n - \theta^{(0)} \\ \hat\beta_n - \beta^{(0)}\end{pmatrix},
\]
where $\bar\nabla$ is the sample Jacobian evaluated at an intermediate point $(\theta^*_n, \beta^*_n)$ between the estimates and the truth.  By the uniform law of large numbers, $\bar\nabla \xrightarrow{p} \mathcal{A}$.  Rearranging:
\begin{equation}\label{eq:joint-expansion}
\sqrt{n}\begin{pmatrix}\hat\theta_n - \theta^{(0)} \\ \hat\beta_n - \beta^{(0)}\end{pmatrix}
\;=\;
-\mathcal{A}^{-1}\;\frac{1}{\sqrt{n}}\sum_{i=1}^n G_i(\theta^{(0)},\beta^{(0)}) \;+\; o_p(1).
\end{equation}
The sum $n^{-1/2}\sum_i G_i(\theta^{(0)},\beta^{(0)})$ is a standardised average of iid mean-zero random vectors.  Under the finite-variance condition (our Assumption~\ref{ass:regularity}(iii)), the multivariate CLT gives
\[
\frac{1}{\sqrt{n}}\sum_{i=1}^n G_i(\theta^{(0)},\beta^{(0)}) \;\xrightarrow{d}\; \mathcal{N}(0, \mathcal{B}),
\]
and the result in Eq.~\eqref{eq:zz-joint-normality} follows by Slutsky's theorem.
\end{proof}

\subsection{Binary treatment and $h_1=T,h_2=M$}

In our model, $h_1(T,M,X) = T$, $h_2(T,M,X) = M$, so $H_i = (T_i, M_i)^\top$.  The weight from Eq.~\eqref{eq:optimal-weight} is $A(T_i,X_i) = (T_i - p)\,W(X_i)$ with $W(X_i) = (1,\;\tau_M(X_i))^\top$.  Since $\E[T_i - p \mid X_i] = 0$, the centred weight $\tilde A_i = A_i - \E[A_i\mid X_i] = A_i$.

\paragraph{Computing the Jacobian $\theta$-block.}
\begin{align}
\frac{\partial G_{1i}}{\partial\theta} &= -\tilde A_i^\top\,H_i^\top \;=\; -(T_i - p)\,W(X_i)\,H_i^\top \nonumber\\
&= -(T_i - p)\begin{pmatrix}1 \\ \tau_M(X_i)\end{pmatrix}(T_i,\; M_i). \label{eq:dG1-dtheta}
\end{align}
Taking expectations and using $\E[(T_i - p)\mid X_i] = 0$:
\begin{align}
\E\!\left[\frac{\partial G_{1i}}{\partial\theta}\right]
&= -\E\!\big[(T_i - p)\,W(X_i)\,H_i^\top\big] \nonumber\\
&= -\E\!\left[(T_i - p)\begin{pmatrix} T_i & M_i \\ \tau_M(X_i)\,T_i & \tau_M(X_i)\,M_i \end{pmatrix}\right], \label{eq:bread-intermediate}
\end{align}
Since $H_i = (T_i, M_i)^\top$ and $M_i = \tau(X_i)T_i + \mu_0(X_i) + K_i + \varepsilon_{M,i}$, we compute column by column.

\emph{Column~1 (derivative w.r.t.\ $\theta_1$):}  The $(j,1)$ entry of $G_\theta$ is $\E[(T_i-p)\,W_j(X_i)\,T_i]$.  For $j=1$: $\E[(T_i-p)\,T_i] = \Var(T) = p(1-p)$.  For $j=2$: $\E[(T_i-p)\,\tau_M(X_i)\,T_i] = p(1-p)\,\E[\tau_M(X_i)]$ (since $T\indep X$).

\emph{Column~2 (derivative w.r.t.\ $\theta_2$):}  The $(j,2)$ entry is $\E[(T_i-p)\,W_j(X_i)\,M_i]$.  Substituting the mediator model and using $T\indep X$:
\begin{align}
\E[(T_i - p)\,W_j(X_i)\,M_i]
&= \E\big[(T_i - p)\,W_j(X_i)\big\{\tau(X_i)\,T_i + \mu_0(X_i) + K_i + \varepsilon_{M,i}\big\}\big] \nonumber\\
&= \E\big[(T_i - p)^2\big]\,\E[W_j(X_i)\,\tau(X_i)] + 0, \label{eq:col2-computation}
\end{align}
where the terms involving $\mu_0, K, \varepsilon_M$ vanish because $\E[(T_i-p)\mid X_i] = 0$ and these terms do not depend on~$T_i$ beyond $X_i$.  Specifically:
\begin{align*}
\E[(T_i - p)\,W_j(X_i)\,\mu_0(X_i)] &= \E\!\big[\E[(T_i-p)\mid X_i]\,W_j(X_i)\,\mu_0(X_i)\big] = 0, \\
\E[(T_i - p)\,W_j(X_i)\,K_i] &= \E[(T_i-p)\,W_j(X_i)]\,\E[K_i] = 0 \quad(K\indep(X,T)).
\end{align*}

Hence, the ``bread'' matrix is
\begin{equation}\label{eq:Gtheta-explicit}
G_\theta \;=\; -p(1-p)\begin{pmatrix}
1 & \E[\tau_M(X)] \\
\E[\tau_M(X)] & \E[\tau_M(X)^2]
\end{pmatrix}
\;=\; -p(1-p)\,\E[W(X)\,W(X)^\top].
\end{equation}
This can be written compactly as
\begin{equation}\label{eq:Gtheta-compact}
G_\theta \;=\; -\E\!\big[(T-p)^2\,W(X)\,W(X)^\top\big],
\end{equation}
using $\Var(T) = p(1-p)$ and $T\indep X$.

\paragraph{Computing the score variance.}
At the true parameters, the score is
\[
G_{1i}(\theta^{(0)},\beta^{(0)}) \;=\; \tilde A_i^\top\,\zeta_{Y,i} \;=\; (T_i - p)\,W(X_i)\,\zeta_{Y,i},
\]
where $\zeta_{Y,i} = Y_i - g_0(X_i) - H_i^\top\theta^{(0)}$ is the structural residual.  Let $\sigma^2_\zeta(X) = \Var(\zeta_{Y,i}\mid X_i)$ be the conditional residual variance.  The ``meat'' matrix is
\begin{equation}\label{eq:meat-explicit}
S \;=\; \E\!\big[(T_i - p)^2\,W(X_i)\,W(X_i)^\top\,\zeta_{Y,i}^2\big].
\end{equation}
Using $T\indep (X, \zeta_Y)$ (which holds because $T$ is randomised, so $T\indep X$, and $\zeta_Y = \lambda_K K + \varepsilon_{Y}$ is a function of unobservables independent of~$T$):
\begin{equation}
S \;=\; p(1-p)\,\E\!\big[W(X)\,W(X)^\top\,\zeta_{Y,i}^2\big].
\end{equation}

\subsection{The variance proportionality and the role of weight quality}

Combining the bread from Eq.~\eqref{eq:Gtheta-compact} and meat from Eq.~\eqref{eq:meat-explicit} into the sandwich, and noting $G_\theta^{-1} = -[p(1-p)]^{-1}\,\E[WW^\top]^{-1}$:
\begin{align}
\Sigma_\theta
&= G_\theta^{-1}\,S\,G_\theta^{-\top} \nonumber\\
&= \frac{1}{[p(1-p)]^2}\;\E[WW^\top]^{-1}\;p(1-p)\,\E[WW^\top\zeta_Y^2]\;\E[WW^\top]^{-1} \nonumber\\
&= \frac{1}{p(1-p)}\;\E[WW^\top]^{-1}\;\E[WW^\top\zeta_Y^2]\;\E[WW^\top]^{-1}. \label{eq:full-sandwich}
\end{align}

When the conditional residual variance is approximately constant, $\zeta_{Y,i}^2 \approx \sigma^2_\zeta$, the meat simplifies to $S \approx p(1-p)\,\sigma^2_\zeta\,\E[WW^\top]$, and the sandwich reduces to
\begin{equation}\label{eq:homoskedastic-sandwich}
\Sigma_\theta \;\approx\; \frac{\sigma^2_\zeta}{p(1-p)}\;\left(\E[W(X)\,W(X)^\top]\right)^{-1}.
\end{equation}
Because $T \indep X$, we have $\E[W(X)\,W(X)^\top\,(T-p)^2] = p(1-p)\,\E[W(X)\,W(X)^\top]$, so Eq.~\eqref{eq:homoskedastic-sandwich} coincides (for the oracle weight $W=W^*$) with the proportionality stated in the main text:
\begin{equation}\label{eq:avar-proportionality}
\Avar(\hat\theta) \;\propto\; \left(\E\!\left[W(X)\,W(X)^\top\,(T-p)^2\right]\right)^{-1}.
\end{equation}

\subsection{Attenuation mechanism}

The key $2\times 2$ matrix $\E[WW^\top]$ has the structure
\begin{equation}\label{eq:WW-matrix}
\E[W(X)\,W(X)^\top] \;=\; \begin{pmatrix}
1 & \E[\tau_M(X)] \\
\E[\tau_M(X)] & \E[\tau_M(X)^2]
\end{pmatrix}.
\end{equation}
Its determinant is $\E[\tau_M^2] - (\E[\tau_M])^2 = \Var(\tau_M(X))$, which is strictly positive by Assumption~\ref{ass:nondegen}.  Inverting the $2\times 2$ system, the $(2,2)$ entry of $\E[WW^\top]^{-1}$ (governing $\Var(\hat\theta_2)$) is
\begin{equation}\label{eq:theta2-var}
\big[\E[WW^\top]^{-1}\big]_{22} \;=\; \frac{1}{\Var(\tau_M(X))}.
\end{equation}
Under the homoskedastic approximation in Eq.~\eqref{eq:homoskedastic-sandwich}, the oracle-weight ($W^*=(1,\tau_M)^\top$) asymptotic variance is
\begin{equation}\label{eq:theta2-avar}
\Avar(\hat\theta_2) \;\approx\; \frac{\sigma^2_\zeta}{p(1-p)\,\Var(\tau_M(X))}.
\end{equation}

In practice Stage~1 supplies a plug-in weight $\hat W(X) = (1,\;\hat\tau_M(X))^\top$ in place of $W^*$.  Because the basis $H = (T, M)^\top$ still contains the true mediator, the sandwich bread is the asymmetric cross-moment $\hat G_\theta = -p(1-p)\,\E[\hat W\, W^{*\top}]$ rather than $\E[\hat W\hat W^\top]$, while the meat remains $p(1-p)\,\sigma^2_\zeta\,\E[\hat W\hat W^\top]$.  Carrying both through the just-identified sandwich $\hat G_\theta^{-1}\hat S\,\hat G_\theta^{-\top}$ and extracting the $\theta_2$ block gives
\begin{equation}\label{eq:theta2-avar-feasible}
\Avar(\hat\theta_2) \;\approx\; \frac{\sigma^2_\zeta}{p(1-p)\,\Var(\tau_M(X))}\;\cdot\;\frac{1}{\Corr(\hat\tau_M,\tau_M)^2},
\end{equation}
so the efficiency loss relative to the oracle in Eq.~\eqref{eq:theta2-avar} is the factor $1/\Corr(\hat\tau_M,\tau_M)^2$.  Two consequences follow.  First, a pure rescaling $\hat\tau_M = \alpha\,\tau_M$ ($\alpha \neq 0$) leaves $\Corr(\hat\tau_M,\tau_M) = 1$ and therefore does not change $\Avar(\hat\theta_2)$: the scale of $\hat\tau_M$ cancels between bread and meat, and only its alignment with $\tau_M$ is relevant.  Second, in the degenerate limit $\hat\tau_M(x) \equiv c$ (so $\Corr(\hat\tau_M,\tau_M) \to 0$), the bread $\E[\hat W\, W^{*\top}]$ becomes singular and $\theta_2$ is not identified.

This is the formal basis for the representation gap: the efficiency loss for~$\hat\theta_2$ is governed by how well the learned representation~$\Phi$ aligns $\hat\tau_M$ with $\tau_M$, measured by $\Corr(\hat\tau_M,\tau_M)^2$ (equivalently, the gap $1 - \Corr(\hat\tau_M,\tau_M)^2$), not by the variance ratio $\Var(\tau_M)/\Var(\hat\tau_M)$.  This is precisely the quantity reported empirically in Section~\ref{sec:part2-results}.

\section{Inference and Asymptotic Properties}
\label{app:inference}

We provide a complete proof of the asymptotic normality of the UNIT estimator, extending the Zheng--Zhou Theorem~2 framework (Appendix~\ref{app:rep-gap}) to accommodate the two-stage structure with cross-fitted first-stage estimates.

\subsection{Notation and setup}

Denote the joint instrument--covariate vector by $C_i = (\hat A_i^\top, B_i^\top)^\top$, where $\hat A_i = (T_i - p)\hat W(X_i)$ is the centred weight and $B(X_i)$ is the instrument for the baseline parameters (e.g., $B_i = X_i$ for linear~$g$).  Let $\xi_i = (H_i^\top, X_i^\top)^\top$ collect all predictors of the outcome, so that the joint estimating equation is $\sum_i C_i(Y_i - \xi_i^\top\gamma) = 0$ with $\gamma = (\theta^\top, \beta^\top)^\top$.

\subsection{Regularity conditions}

\begin{assumption}[Regularity conditions]\label{ass:regularity}
\mbox{}
\begin{enumerate}
    \item[(i)] (IID sampling) The observations $(X_i, T_i, M_i, Y_i)$ are iid draws from a common distribution~$P$.
    \item[(ii)] (Moment condition) The ``bread'' matrix $G = \E[C_i\,\xi_i^\top]$ is nonsingular (cf.\ Assumption~6 in \citet{ZhengZhou2015}).
    \item[(iii)] (Finite variance) $\E[\|C_i\|^2\,\zeta_{Y,i}^2] < \infty$, where $\zeta_{Y,i} = Y_i - g(X_i) - H_i^\top\theta_0$ is the structural residual.
    \item[(iv)] (First-stage rate) The cross-fitted mediator CATE satisfies $\|\hat\tau_M^{(-)} - \tau_M\|_{L_2(P)} = o_p(n^{-1/4})$.
    \item[(v)] (Baseline consistency) The cross-fitted baseline estimator satisfies
    \[
    \sup_{\|\theta-\theta_0\|\le\delta_n} \|\hat g_\theta^{(-)} - g_{0,\theta}\|_{L_2(P)} = o_p(1)
    \]
    for some sequence $\delta_n\downarrow 0$, where $g_{0,\theta}(x) = \E[Y - H^\top\theta\mid X=x]$.
\end{enumerate}
\end{assumption}

\paragraph{Discussion of individual conditions.}

Condition~(i) is standard.  Condition~(ii) requires that the joint instrument vector is sufficiently ``diverse'' to identify all parameters; for the $\theta$-block this reduces to Assumption~\ref{ass:nondegen} ($\Var(\tau_M(X)) > 0$), and for the $\beta$-block it requires that $B(X)$ spans the baseline parameter space.  Condition~(iii) is a moment condition ensuring finite sandwich variance.

Condition~(iv) is the rate requirement for the first stage.  It states that the CATE estimation error must vanish faster than $n^{-1/4}$, which is the threshold for the cross-product remainder in the Neyman-orthogonal score to be $o_p(n^{-1/2})$.  This is a standard DML rate condition \citep{Chernozhukov2018}.

Condition~(v) requires uniform consistency of the cross-fitted baseline in a shrinking neighbourhood of $\theta_0$.  For a linear Ridge estimator with fixed-dimensional $b(X)$, this is satisfied under standard conditions (bounded eigenvalues of $\E[b(X)b(X)^\top]$ and shrinkage $\lambda_n \to 0$).

\subsection{The role of $\Omega^{-1}$ and the linear-$g$ simplification}
\label{subsec:omega-role}

Theorem~3 of \citet{ZhengZhou2015} shows that the fully efficient weight includes the factor $\Omega^{-1}(X) = \{\Var(\varepsilon^{tm}(U,X)\mid X)\}^{-1}$.  The semiparametrically optimal instrument is
\begin{equation}\label{eq:optimal-instrument-general}
A^*(T,X) \;=\; \big\{\E[\partial\tilde Y(\theta)/\partial\theta \mid X,T] - \E[\partial\tilde Y(\theta)/\partial\theta \mid X]\big\}\;\Omega^{-1}(X),
\end{equation}
where the first factor captures the heterogeneity in the structural derivative across treatment arms, and $\Omega^{-1}$ re-weights observations inversely to their residual variance.

In the present paper we adopt a linear specification $g(X) = \beta^\top b(X)$, under which, as noted by \cite{Morelli2025}, the estimation simplifies to a single step and $\Omega^{-1}$ need not be estimated.  The plug-in weight reduces to $\hat A_i = (T_i - p)\hat W(X_i)$ with $\hat W$ fixed from Stage~1.  Below we make the linear-$g$ simplification precise and then discuss, for completeness, how the asymptotic argument would change if~$g$ were modelled nonlinearly.

\paragraph{Why $\Omega^{-1}$ drops out under linear~$g$.}
Under our model with $h_1 = T$, $h_2 = M$ and a linear working model $g(X) = \beta^\top b(X)$, the joint system in Eq.~\eqref{eq:stacked-moments} is linear in $\gamma = (\theta^\top, \beta^\top)^\top$.  The $G_2$-equation $\sum_i B_i(Y_i - H_i^\top\theta - b(X_i)^\top\beta) = 0$ has the closed-form solution
\[
\hat\beta(\theta) \;=\; \Bigl(\sum_i B_i\,b(X_i)^\top\Bigr)^{-1}\sum_i B_i\,(Y_i - H_i^\top\theta),
\]
so $\hat\beta$ is a deterministic affine function of~$\theta$.  The converged iterate $(\hat\theta, \hat\beta)$ is therefore the unique $M$-estimator of the joint linear system, and the full sandwich variance of the $\theta$-block (Proposition~\ref{prop:normality}) is valid without estimating~$\Omega^{-1}$.  Omitting $\Omega^{-1}$ from the weight means that the estimator is not semiparametrically efficient in general, it uses the instrument $(T-p)\,W(X)$ rather than $(T-p)\,W(X)\,\Omega^{-1}(X)$,but it remains consistent and asymptotically normal.  

\paragraph{What changes under nonlinear~$g$.}
If the baseline were modelled nonlinearly, $g(X) = g(X;\beta)$ with $g$ a neural network or other flexible learner, we recognize that at least three aspects of the argument would be affected.  The identification of~$\theta$ and the centering property $\E[\partial G_1/\partial\beta] = 0$ would be unaffected, since they depend only on $\E[\tilde A_i\mid X_i] = 0$ and on $g(X;\beta)$ being a function of~$X_i$ alone.  However:

\begin{enumerate}
\item \emph{Loss of linearity in the joint system.}  With nonlinear~$g$, the $G_2$-equation
\[
\sum_i B_i\big(Y_i - H_i^\top\theta - g(X_i;\beta)\big) = 0
\]
is nonlinear in~$\beta$.  The $\beta$-update no longer has a closed form: it requires an inner optimisation loop (e.g., gradient descent), and convergence of the alternating scheme to a fixed point requires additional regularity (e.g., a contraction condition on the map $\theta \mapsto \hat\beta(\theta) \mapsto \hat\theta(\hat\beta)$).

\item \emph{$\Omega^{-1}$ must be estimated.}  The semiparametrically efficient weight from Eq.~\eqref{eq:optimal-instrument-general} includes $\Omega^{-1}(X)$, which must be estimated from the squared residuals $\{Y_i - \hat g(X_i;\hat\beta) - H_i^\top\hat\theta\}^2$.  Since these residuals depend on the current~$\hat\theta$, an outer iteration between $\hat\theta$ and $\hat\Omega$ is needed.  In practice, \citet{ZhengZhou2015} recommend a single one-step update rather than iterating to convergence, noting that a misspecified $\Omega$ does not affect consistency of~$\hat\theta$ but only its efficiency.

\item \emph{Uniqueness of~$\beta^{(0)}$ becomes non-trivial.}
Proposition 1 assumes that there exists a unique solution $\beta^{(0)}$ to the population baseline equation $\E[G_{2i}(\theta^{(0)},\beta)] = 0$.  As \citet{ZhengZhou2015} remark explicitly: ``When we use a linear working model, this assumption always holds.  However, when we use a non-linear model, it is difficult to check whether this assumption holds.''  The reason is as follows.  With linear~$g$, the equation $\E[B_i\,(Y_i - H_i^\top\theta_0 - b(X_i)^\top\beta)] = 0$ is affine in~$\beta$, and uniqueness reduces to the standard rank condition $\E[B_i\,b(X_i)^\top]$ nonsingular.  With nonlinear $g(X;\beta)$, the equation
\[
\E\!\big[B_i\,\big(Y_i - H_i^\top\theta_0 - g(X_i;\beta)\big)\big] = 0
\]
defines a nonlinear system in~$\beta$ that may have multiple roots, no root, or a root that is a saddle point of the population objective.  \citet{ZhengZhou2015} suggest a partial diagnostic: if the regression model $\E[Y\mid X] = g(X;\beta)$ is not identified (i.e., two distinct parameter values yield the same conditional mean), then the unique-solution assumption fails.  The practical implication is that choosing a nonlinear~$g$ requires verifying that the parametric model $g(X;\beta)$ is globally identified. In case the parametric model is not globally identified, the alternating updates of $g(X;\beta)$ and $\theta$ are not guaranteed to converge.

\end{enumerate}

In summary, the identification and centering arguments are entirely unaffected by the functional form of~$g$: what changes is the estimation procedure (closed-form versus iterative), the necessity of estimating~$\Omega^{-1}$, and the structural condition of uniqueness of~$\beta^{(0)}$ .  Our current linear-$g$ specification avoids all these complications, which is why we have adopted it as the baseline implementation.  Extending UNIT to nonlinear~$g$ with formal inferential guarantees is a concrete direction for future work, requiring the addition of a third cross-fitted nuisance ($\hat\Omega$) and a global identifiability argument for the working baseline model.

\subsection{Asymptotic Normality}

\begin{proposition}[Asymptotic normality]\label{prop:normality}
Under Assumptions~\ref{ass:sutva}--\ref{ass:nondegen} and the regularity conditions (Assumption~\ref{ass:regularity}), the UNIT G-estimator $\hat\theta$ satisfies
\[
\sqrt{n}\,(\hat\theta - \theta_0) \;\xrightarrow{d}\; \mathcal{N}\!\left(0,\; G_\theta^{-1}\,S\,G_\theta^{-\top}\right),
\]
where $G_\theta$ is the $\theta$-block of the Jacobian
\begin{equation}\label{eq:bread}
G_\theta \;=\; \E\!\left[\hat A_i\,H_i^\top\right],
\end{equation}
and $S$ is the variance of the influence function evaluated at the true parameters,
\begin{equation}\label{eq:meat}
S \;=\; \Var\!\left[\hat A_i\,\zeta_{Y,i}\right] \;=\; \E\!\left[\hat A_i\,\hat A_i^\top\,\zeta_{Y,i}^2\right], \qquad \zeta_{Y,i} = Y_i - g_0(X_i) - H_i^\top\theta_0,
\end{equation}
where the second equality uses $\E[\hat A_i\,\zeta_{Y,i}]=0$.
\end{proposition}

\begin{proof}
We proceed in four steps: (1) identify the UNIT estimator as an $M$-estimator of a stacked moment condition; (2) perform a first-order Taylor expansion; (3) bound the two remainder terms arising from plugging in the estimated first-stage CATE and the estimated baseline; and (4) apply the CLT to the leading term.

\medskip
\textbf{Step~1: $M$-estimation formulation.}

The UNIT algorithm (Algorithm~\ref{alg:unit}) alternates between updating the baseline $\hat g_\theta^{(-)}$ and updating the structural parameter~$\theta$.  At convergence, the pair $(\hat\theta, \hat\beta)$ jointly solves
\begin{equation}\label{eq:stacked-moments}
\sum_{i=1}^n \begin{pmatrix} \hat A_i\!\left(Y_i - \hat g_{\hat\theta}^{(-)}(X_i) - H_i^\top\hat\theta\right) \\[4pt] B_i\!\left(Y_i - \hat g_{\hat\theta}^{(-)}(X_i) - H_i^\top\hat\theta\right) \end{pmatrix}
\;=\; \begin{pmatrix} 0 \\ 0 \end{pmatrix}.
\end{equation}
For the linear baseline $\hat g_\theta^{(-)}(X_i) = b(X_i)^\top\hat\beta$, this is exactly the stacked system $\sum_i C_i(Y_i - \xi_i^\top\hat\gamma) = 0$ with $C_i = (\hat A_i^\top, B_i^\top)^\top$ and $\xi_i = (H_i^\top, b(X_i)^\top)^\top$.

The key complication relative to the classical Zheng--Zhou setting is that $\hat A_i$ depends on the estimated first-stage quantity $\hat\tau_M^{(-)}(X_i)$, which is itself a function of the data.  We decompose the analysis into the oracle component (as if $\tau_M$ were known) and two remainder terms.

\medskip
\textbf{Step~2: Taylor expansion.}

Write $\hat A_i = (T_i - p)\hat W(X_i)$ with $\hat W(X_i) = (1,\;\hat\tau_M^{(-)}(X_i))^\top$, and define the oracle counterpart $A_i^* = (T_i - p)W^*(X_i)$ with $W^*(X_i) = (1,\;\tau_M(X_i))^\top$.  Define
\[
\Delta_i \;=\; \hat A_i - A_i^* \;=\; (T_i - p)\begin{pmatrix}0 \\ \hat\tau_M^{(-)}(X_i) - \tau_M(X_i)\end{pmatrix}.
\]

The oracle estimating equation at the true $\gamma_0 = (\theta_0^\top, \beta_0^\top)^\top$ has score
\[
\psi_i^* \;=\; C_i^*\,(Y_i - \xi_i^\top\gamma_0) \;=\; C_i^*\,\zeta_{Y,i},
\]
where $C_i^* = (A_i^{*\top}, B_i^\top)^\top$.  By construction, $\E[\psi_i^*] = 0$.

The estimated score decomposes as
\begin{align}
C_i\,\zeta_{Y,i}
&= C_i^*\,\zeta_{Y,i} + \underbrace{(\Delta_i^\top, 0)^\top\,\zeta_{Y,i}}_{\text{first-stage perturbation}}. \label{eq:score-decomposition}
\end{align}

The sample moment condition at $\hat\gamma$ is $n^{-1}\sum_i C_i(Y_i - \xi_i^\top\hat\gamma) = 0$.  Adding and subtracting $\gamma_0$:
\begin{align}
0 &= \frac{1}{n}\sum_{i=1}^n C_i\,\zeta_{Y,i} \;-\; \frac{1}{n}\sum_{i=1}^n C_i\,\xi_i^\top\,(\hat\gamma - \gamma_0) \nonumber\\
&= \frac{1}{n}\sum_{i=1}^n C_i^*\,\zeta_{Y,i}
\;+\; \frac{1}{n}\sum_{i=1}^n \begin{pmatrix}\Delta_i \\ 0\end{pmatrix}\zeta_{Y,i}
\;-\; \hat G\,(\hat\gamma - \gamma_0), \label{eq:full-taylor}
\end{align}
where $\hat G = n^{-1}\sum_i C_i\,\xi_i^\top \xrightarrow{p} G = \E[C_i^*\,\xi_i^\top]$ (the perturbation from $\Delta_i$ in $\hat G$ is $o_p(1)$ under Assumption~\ref{ass:regularity}(iv)).  Solving:
\begin{equation}\label{eq:gamma-expansion}
\sqrt{n}(\hat\gamma - \gamma_0)
\;=\;
G^{-1}\,\frac{1}{\sqrt{n}}\sum_{i=1}^n C_i^*\,\zeta_{Y,i}
\;+\; \underbrace{G^{-1}\,\frac{1}{\sqrt{n}}\sum_{i=1}^n \binom{\Delta_i}{0}\zeta_{Y,i}}_{R_\tau}
\;+\; \underbrace{R_g}_{(\text{from }\hat g)}
\;+\; o_p(1),
\end{equation}
where $R_g$ arises from the difference $\hat g_\theta^{(-)} - g_{0,\theta}$.

\medskip
\textbf{Step~3a: Bounding $R_\tau$ (first-stage remainder).}

The $\theta$-block of $R_\tau$ is
\[
R_\tau^{(\theta)} \;=\; G_\theta^{-1}\,\frac{1}{\sqrt{n}}\sum_{i=1}^n \Delta_i\,\zeta_{Y,i}
\;=\; G_\theta^{-1}\,\frac{1}{\sqrt{n}}\sum_{i=1}^n (T_i - p)\begin{pmatrix}0 \\ \hat\tau_M^{(-)}(X_i) - \tau_M(X_i)\end{pmatrix}\zeta_{Y,i}.
\]
Let
\[
\eta_i \;=\; (T_i - p)\,\big(\hat\tau_M^{(-)}(X_i) - \tau_M(X_i)\big)\,\zeta_{Y,i}
\]
denote the scalar summand entering the second component of $R_\tau^{(\theta)}$.  We show that $n^{-1/2}\sum_i \eta_i = o_p(1)$; this implies $R_\tau^{(\theta)} = o_p(1)$ since $G_\theta^{-1}$ is bounded by Assumption~\ref{ass:nondegen}.

\emph{(a) Conditional mean-zero of $\eta_i$.}
Under randomisation (Assumption~\ref{ass:ignorability}), $T_i \indep U_i \mid X_i$.  Combined with the structural model $Y_i = g_0(X_i) + H_i^\top\theta_0 + \varepsilon(T_i,M_i;U_i,X_i)$, this gives $\E[\zeta_{Y,i}\mid T_i,X_i] = \E[\varepsilon(T_i,M_i;U_i,X_i)\mid T_i,X_i] = \E[\varepsilon\mid X_i]$, which does not depend on~$T_i$.  Hence
\begin{equation}\label{eq:tperp-zeta}
\E[(T_i - p)\,\zeta_{Y,i}\mid X_i] \;=\; \E[\zeta_{Y,i}\mid X_i]\;\E[(T_i - p)\mid X_i] \;=\; 0.
\end{equation}
Because $\hat\tau_M^{(-)}(X_i)$ is measurable with respect to the training-fold data $I_k^c$, and the held-out observation $(T_i, M_i, Y_i, X_i)$ for $i\in I_k$ is independent of $I_k^c$, Eq.~\eqref{eq:tperp-zeta} yields
\begin{equation}\label{eq:eta-cond-zero}
\E\!\big[\eta_i \,\big|\, X_i,\;\hat\tau_M^{(-)},\;I_k^c\big] \;=\; \big(\hat\tau_M^{(-)}(X_i) - \tau_M(X_i)\big)\,\E[(T_i - p)\,\zeta_{Y,i}\mid X_i] \;=\; 0.
\end{equation}

\emph{(b) Variance bound.}
Conditional on the training folds $\{I_k^c\}_k$, the $\eta_i$ are mutually independent across $i$ with conditional mean zero, so
\begin{align}
\Var\!\Bigl(\tfrac{1}{\sqrt{n}}\sum_{i=1}^n \eta_i \,\Bigm|\, \{I_k^c\}_k\Bigr)
&\;=\; \tfrac{1}{n}\sum_{i=1}^n \E\!\big[\eta_i^2 \,\big|\, \{I_k^c\}_k\big] \nonumber\\
&\;=\; \E\!\big[(\hat\tau_M^{(-)}(X)-\tau_M(X))^2\,(T-p)^2\,\zeta_{Y}^2 \,\big|\, \{I_k^c\}_k\big] \nonumber\\
&\;\lesssim\; \|\zeta_Y\|_{\infty}^{2}\;\Var(T)\;\big\|\hat\tau_M^{(-)} - \tau_M\big\|_{L_2(P)}^{2}, \label{eq:Rtau-var-bound}
\end{align}
where the last inequality uses the moment condition in Assumption~\ref{ass:regularity}(iii).

\emph{(c) Conclusion.}
By the conditional Chebyshev inequality and Eq.~\eqref{eq:Rtau-var-bound},
\begin{equation}\label{eq:Rtau-order}
\Bigl|\tfrac{1}{\sqrt{n}}\sum_{i=1}^n \eta_i\Bigr| \;=\; O_p\!\big(\|\hat\tau_M^{(-)} - \tau_M\|_{L_2(P)}\big),
\end{equation}
so under Assumption~\ref{ass:regularity}(iv) we obtain $R_\tau^{(\theta)} = o_p(1)$.

\medskip
\textbf{Step~3b: Bounding $R_g$ (baseline remainder).}

The baseline remainder arises from $\hat g_\theta^{(-)}(X_i) \neq g_{0,\theta}(X_i)$.  The score contribution from unit~$i$ includes $\hat A_i\,(\hat g_\theta^{(-)}(X_i) - g_{0,\theta}(X_i))$.  Focusing on the $\theta$-block:
\[
R_g^{(\theta)} \;=\; G_\theta^{-1}\,\frac{1}{\sqrt{n}}\sum_{i=1}^n \hat A_i\,\big(\hat g_\theta^{(-)}(X_i) - g_{0,\theta}(X_i)\big).
\]

For each $i \in I_k$, the estimate $\hat g_\theta^{(-k)}$ is trained on $I_k^c$.  Therefore, conditional on $I_k^c$, the held-out observations $(X_i, T_i, M_i, Y_i)$ for $i\in I_k$ are independent of $\hat g_\theta^{(-k)}$.  The conditional expectation is
\[
\E\!\big[\hat A_i\,(\hat g_\theta^{(-k)}(X_i) - g_{0,\theta}(X_i))\;\big|\;I_k^c\big]
= \E\!\big[(T_i - p)\,\hat W(X_i)\,(\hat g_\theta^{(-k)}(X_i) - g_{0,\theta}(X_i))\;\big|\;I_k^c\big].
\]

Since $\hat W(X_i) = (1,\;\hat\tau_M^{(-k)}(X_i))^\top$ is also measurable with respect to $I_k^c$, we have, conditional on $I_k^c$:
\[
\E\!\big[(T_i - p)\,\hat W(X_i)\,(\hat g_\theta^{(-k)}(X_i) - g_{0,\theta}(X_i))\;\big|\;X_i,\;I_k^c\big]
= \hat W(X_i)\,(\hat g_\theta^{(-k)}(X_i) - g_{0,\theta}(X_i))\;\E[(T_i-p)\mid X_i] = 0.
\]
The final equality uses the centering $\E[(T_i-p)\mid X_i] = 0$ from randomisation.  This is the population-level orthogonality \citep[the working-model property of][]{ZhengZhou2015}), but here it holds \emph{conditionally} on $I_k^c$ due to the independence provided by cross-fitting.

Therefore, $n^{-1/2}\sum_i \hat A_i(\hat g_\theta^{(-)}(X_i) - g_{0,\theta}(X_i))$ is a sum of conditionally mean-zero terms.  By the conditional Chebyshev inequality:
\begin{equation}\label{eq:Rg-bound}
\|R_g\| \;=\; O_p\!\left(\|\hat g_\theta^{(-)} - g_{0,\theta}\|_{L_2}\right) \;=\; o_p(1)
\end{equation}
under Assumption~\ref{ass:regularity}(v).





\medskip
\textbf{Step~4: CLT}

With both remainders shown to be $o_p(1)$, Eq.~\eqref{eq:gamma-expansion} simplifies to
\begin{equation}\label{eq:leading-term}
\sqrt{n}(\hat\gamma - \gamma_0) \;=\; G^{-1}\,\frac{1}{\sqrt{n}}\sum_{i=1}^n C_i^*\,\zeta_{Y,i} \;+\; o_p(1).
\end{equation}

The leading term is a standardised sum of iid random vectors with mean zero (since $\E[C_i^*\,\zeta_{Y,i}] = 0$) and finite covariance (Assumption~\ref{ass:regularity}(iii)).  By the multivariate central limit theorem:
\[
\frac{1}{\sqrt{n}}\sum_{i=1}^n C_i^*\,\zeta_{Y,i} \;\xrightarrow{d}\; \mathcal{N}\!\big(0,\;\E[C_i^*\,C_i^{*\top}\,\zeta_{Y,i}^2]\big).
\]

By Slutsky's theorem:
\[
\sqrt{n}(\hat\gamma - \gamma_0) \;\xrightarrow{d}\; \mathcal{N}\!\left(0,\; G^{-1}\,\E[C_i^*\,C_i^{*\top}\,\zeta_{Y,i}^2]\,G^{-\top}\right).
\]

Restricting to the $\theta$-block and using the block structure of~$G$ (recalling that the $(\theta,\beta)$ cross-block vanishes by the centering property from Eq.~\eqref{eq:centering-key}):
\[
\sqrt{n}(\hat\theta - \theta_0) \;\xrightarrow{d}\; \mathcal{N}\!\left(0,\; G_\theta^{-1}\,S\,G_\theta^{-\top}\right),
\]
where $G_\theta = \E[\hat A_i\,H_i^\top] = \E[A_i^*\,H_i^\top]$ (the oracle bread, since the perturbation from $\Delta_i$ in $G_\theta$ is $o_p(1)$) and $S = \E[A_i^*\,A_i^{*\top}\,\zeta_{Y,i}^2] = \Var[A_i^*\,\zeta_{Y,i}]$.
\end{proof}

\subsection{Sandwich variance estimation}

A consistent estimator of the asymptotic variance is obtained by replacing population quantities with sample analogues evaluated at the estimated parameters:
\begin{equation}\label{eq:sandwich-consistent}
\widehat{\Var}(\hat\theta) \;=\; \frac{1}{n}\,\hat G_\theta^{-1}\,\hat S\,\hat G_\theta^{-\top},
\end{equation}
where
\[
\hat G_\theta \;=\; \frac{1}{n}\sum_{i=1}^n \hat A_i\,H_i^\top, \qquad
\hat S \;=\; \frac{1}{n}\sum_{i=1}^n \hat\psi_i^2\,\hat A_i\,\hat A_i^\top,
\]
with out-of-fold residuals $\hat\psi_i = Y_i - \hat g_{\hat\theta}^{(-)}(X_i) - H_i^\top\hat\theta$.

\begin{remark}[Consistency of the sandwich estimator]
By the uniform law of large numbers and consistency of $(\hat\theta, \hat\beta)$:
\begin{enumerate}
\item[(a)] $\hat G_\theta \xrightarrow{p} G_\theta$, because $\hat A_i \to A_i^*$ in $L_2$ (Assumption~\ref{ass:regularity}(iv)) and $H_i$ is bounded.
\item[(b)] $\hat\psi_i \xrightarrow{p} \zeta_{Y,i}$ for each~$i$, by consistency of $\hat\theta$ and $\hat g$.
\item[(c)] Hence $\hat S \xrightarrow{p} S$, and $\widehat{\Var}(\hat\theta) \xrightarrow{p} n^{-1}\,G_\theta^{-1}\,S\,G_\theta^{-\top}$.
\end{enumerate}
The sandwich estimator is therefore consistent.  Moreover, it treats $\hat\tau_M^{(-)}$ and $\hat g_\theta^{(-)}$ as fixed, which is asymptotically justified because their estimation errors are $o_p(n^{-1/2})$ (for $\hat g$) and $o_p(n^{-1/4})$ (for $\hat\tau_M$), both absorbed by the remainder analysis.
\end{remark}

\subsection{Interpretation of the structural parameters}
\label{subsec:estimand}

Under Assumption~\ref{ass:model-spec}, the structural mean model in Eq.~\eqref{eq:smm} defines two parameters with standard causal interpretations.


\paragraph{Controlled effects.}
The parameter~$\theta_1$ is the \emph{controlled direct effect} (CDE): the average change in the potential outcome $Y(t,m)$ when the treatment is switched from $t=0$ to $t=1$ while the mediator is held fixed at level~$m$.  The parameter~$\theta_2$ is the \emph{controlled mediator effect} (CME): the change in $Y(t,m)$ per unit change in~$m$, holding~$t$ fixed.  These are the primary estimands of the structural mean (mediation) model. They are identified under the NEH assumption (Assumption~\ref{ass:neh}) and are directly relevant to intervention design. Specifically, the value of the parameter, $\theta_2$, tells a researcher how much the outcome would change if the mediator could be shifted by one unit. This quantity is needed to evaluate the value of targeting $M$ in future interventions or, if the outcome is delayed in the future, it could potentially be utilised as a surrogate for treatment success \cite{FlemingPowers2012}.

\paragraph{From controlled to natural effects.}
In some cases  $\theta_2\,\tau_M(X)$ can be read as a natural
indirect effect (NIE), defined via cross-world counterfactuals as
\[
\mathrm{NIE}(X)\;=\;\E\!\big[Y(1,M(1))-Y(1,M(0))\mid X\big].
\]
Write the structural residual as
$\varepsilon(t,m)\equiv Y(t,m)-g(X)-\theta_1 t-\theta_2 m$, so that
$Y(1,m)=g(X)+\theta_1+\theta_2 m+\varepsilon(1,m)$. Evaluating at the two
counterfactual mediator values $M(1)$ and $M(0)$ and taking the difference, the baseline
$g(X)+\theta_1$ cancels leaving
\[
Y(1,M(1))-Y(1,M(0))
=\theta_2\big(M(1)-M(0)\big)+\varepsilon(1,M(1))-\varepsilon(1,M(0)).
\]
Taking the conditional expectation given $X$ yields
\begin{equation}\label{eq:nie-decomp}
\E[\mathrm{NIE}(X)\mid X]
=\theta_2\,\tau_M(X)
+\E\!\big[\varepsilon(1,M(1))-\varepsilon(1,M(0))\mid X\big],
\end{equation}
 Thus
$\theta_2\,\tau_M(X)$ equals the NIE up to the residual term in
Eq.~\eqref{eq:nie-decomp}, which measures whether the $M\!\to\!Y$ contrast is modified,
in the mean, by the counterfactual mediator values themselves. We rule this out with
a cross-world condition on $\varepsilon$.
\begin{assumption}[Cross-world residual mean-independence]\label{ass:resid-indep}
For every fixed regime $(t,m)$,
\[
\E\!\big[\varepsilon(t,m)\,\big|\,M(1),M(0),X\big]\;=\;\E\!\big[\varepsilon(t,m)\mid X\big]\;=\;0 .
\]
\end{assumption}
Assumption~\ref{ass:resid-indep} is the mean-level, cross-world form of the ``no unmeasured
confounder that also serves as an effect modifier'' condition of
\citet{ZhengZhou2015}: knowledge of the counterfactual mediators $(M(1),M(0))$
carries no information about the conditional mean of the structural residual at any
fixed regime. Under it the residual term in Eq.~\eqref{eq:nie-decomp} vanishes, since by
iterated expectation
\[
\E[\varepsilon(1,M(1))\mid X]
=\E\big[\,\E[\varepsilon(1,m)\mid M(1),M(0),X]\big|_{m=M(1)}\,\big|\,X\big]=0,
\]
and analogously for $M(0)$. The decomposition then simplifies to:
\begin{equation}\label{eq:nie-bridge}
\E[\mathrm{NIE}(X)\mid X]\;=\;\theta_2\,\E[M(1)-M(0)\mid X]\;=\;\theta_2\,\tau_M(X),
\end{equation}

Assumption~\ref{ass:resid-indep} is not needed to identify $\theta_1$ and $\theta_2$; it is just an additional interpretive condition that bridges the structural parameters to the NIE. We refer to \citet{TenHave2010} and \citet{VANDERWEELENIE} for a comprehensive review of the assumptions required to bridge controlled and natural effects under different model specifications.

\section{Estimator Hyperparameters}
\label{app:hyper}

All Stage~1 estimators are fitted via $K=5$-fold cross-fitting (stratified random split, \texttt{KFold} with \texttt{shuffle=True} and \texttt{random\_state} set to the replication seed).
Stage~2 G-estimation is applied once to the full sample using the out-of-fold $\hat\tau_M(X_i)$ vector produced by Stage~1.

\subsection{Ridge T-Learner (Stage~1)}

A separate \texttt{RidgeCV} model is fitted on the treated ($T=1$) and control ($T=0$) sub-samples.  The regularisation parameter $\alpha$ is chosen independently per arm by $5$-fold cross-validation over the grid $\{0.01,\,0.1,\,1.0,\,10.0\}$.  An intercept is included.
$\hat\tau_M(X) = \hat\mu_1(X) - \hat\mu_0(X)$.
Seed is fixed at $0$ across all replications (the cross-fitting fold split uses the replication seed).

  \subsection{Random Forest T-Learner with Tuning (Stage~1)}

A separate \texttt{RandomForestRegressor} is fitted per arm via \texttt{RandomizedSearchCV} ($n_{\text{iter}}=6$, $3$-fold CV, scoring: negative MSE).  The search grid is:

  \begin{center}
  \begin{tabular}{ll}
  \toprule
  Hyperparameter & Candidates \\
  \midrule
  \texttt{max\_depth}         & $\{10,\;20,\;\text{None}\}$ \\
  \texttt{max\_features}      & $\{0.20,\;0.33,\;0.50\}$ \\
  \texttt{min\_samples\_leaf} & $\{1,\;3,\;5\}$ \\
  \texttt{n\_estimators}      & $\{200,\;300\}$ \\
  \bottomrule
  \end{tabular}
  \end{center}

Seed is fixed at $0$ for the search random state across all replications.
$\hat\tau_M(X) = \hat\mu_1(X) - \hat\mu_0(X)$ using the best-found hyperparameters per arm.

\subsection{ TARNet (Stage~1)}

TARNet \citep{shalit2017} uses a shared representation network $\Phi(X)$ followed by two separate outcome heads $h_0, h_1$, yielding $\hat\tau_M(X) =h_1(\Phi(X)) - h_0(\Phi(X))$.
The architecture and training schedule are scaled to sample size as shown in Table~\ref{tab:tarnet-params}; all other settings are fixed across all $n$.
Seed is set to $(\text{replication seed} + \text{fold index})$ so that each cross-fitting fold receives a distinct initialisation, avoiding correlated failure modes across replications.

  \begin{table}[!ht]
  \centering
  \caption{TARNet hyperparameters by sample size (pooled scenarios).
  $L_r$: representation layers; $d_r$: representation width;
  $L_o$: outcome-head layers; $d_o$: outcome-head width;
  $\lambda$: L2 penalty; $B$: batch size; $P$: early-stopping patience (epochs).
  Validation split is $10\%$ of the training fold in all cases;
  early stopping is not triggered before $20$ full epochs.
  Optimiser: Adam with learning rate $10^{-3}$; activation: ELU.}
  \label{tab:tarnet-params}
  \smallskip
  \begin{tabular}{ccccccccc}
  \toprule
  $n$ & $L_r$ & $d_r$ & $L_o$ & $d_o$ & $\lambda$ & $B$ & $P$ & max iter \\
  \midrule
  ${\le}500$     & 2 & 100 & 1 &  50 & 0.020 &  32 & 20 &  500 \\
  ${\le}2\,000$  & 3 & 200 & 2 & 100 & 0.020 &  64 & 25 &  700 \\
  ${\le}5\,000$  & 3 & 200 & 2 & 100 & 0.010 & 128 & 30 &  800 \\
  ${\le}10\,000$ & 3 & 200 & 2 & 100 & 0.010 & 256 & 40 & 1000 \\
  \bottomrule
  \end{tabular}
  \end{table}

\subsection{G-Estimator (Stage~2)}

The Stage~2 G-estimator follows \citet{ZhengZhou2015} with centred weights $A_i = (T_i - e(X_i))\,(1,\,\hat\tau_M(X_i))^\top$.
The nuisance baseline $g(X)$ is estimated by \texttt{RidgeCV} (regularisation $\alpha\in\{10^{-3},10^{-2},0.1,1,10,100\}$) fitted on $5$ cross-fit folds (same fold partition as Stage~1), so that
regularisation bias in $g$ does not propagate into $\hat\theta$.
The alternating optimization loop runs for at most $100$
iterations with convergence tolerance $\|\theta^{(t+1)}-\theta^{(t)}\|<10^{-4}$.
Standard errors use the analytical sandwich estimator. An experiment is flagged as weakly identified and excluded from SE aggregation when the bread matrix condition number exceeds $10^{6}$ or the Stage~1 NRMSE exceeds $1.3$.

\section{Acknowledgements}

The authors would like to express their gratitude to Prof.Dr. Stijn Vansteelandt  and Dr. Ruicong Yao for their insightful discussions and valuable comments on this work.

\subsection{Use of LLMs}
    
The authors acknowledge: ``Any use of generative AI in this manuscript adheres to ethical guidelines for use and acknowledgment of generative AI in academic research, as outlined in the manuscript \cite{PorsdamMann2024}. Each author has made a substantial contribution to the work, which has been thoroughly vetted for accuracy, and assumes responsibility for the integrity of their contributions. The authors utilized Claude Opus 4.5 for improving the readability and clarity of the text, and Claude Code (Opus 4.5) for supporting the creation, validation and testing of the code.''

\end{document}